\documentclass{article}

\usepackage{arxiv}

\usepackage[utf8]{inputenc} 
\usepackage[T1]{fontenc}    
\usepackage{hyperref}       
\usepackage{url}            
\usepackage{booktabs}       
\usepackage{amsfonts,placeins}       
\usepackage{nicefrac}       
\usepackage{microtype}      
\usepackage{lipsum}
\usepackage[american]{babel}
\usepackage{amsmath,amsthm,amssymb,mathtools}
\usepackage{enumitem,algorithm, algorithmic,qtree,tikz,multirow}
\usepackage{soul}
\usepackage{graphicx}
\usepackage{booktabs}
\usepackage[round]{natbib}
\usepackage{bibentry}
\graphicspath{ {./images/} }
\newcommand\algorithmicprocedure{\textbf{Procedure}}
\newcommand{\algorithmicendprocedure}{\algorithmicend\ \algorithmicprocedure}
\makeatletter
\newcommand\PROCEDURE[3][default]{%
  \ALC@it
  \algorithmicprocedure\ \textsc{#2}(#3)%
  \ALC@com{#1}%
  \begin{ALC@prc}%
}
\newcommand\ENDPROCEDURE{%
  \end{ALC@prc}%
  \ifthenelse{\boolean{ALC@noend}}{}{%
    \ALC@it\algorithmicendprocedure
  }%
}
\newenvironment{ALC@prc}{\begin{ALC@g}}{\end{ALC@g}}
\makeatother
\newenvironment{theorem}[2][Theorem]{\begin{trivlist} \em
\item[\hskip \labelsep {\bfseries #1}\hskip \labelsep {\bfseries #2.}]}{\end{trivlist}}
\newenvironment{remark}[2][Remark]{\begin{trivlist} \em
\item[\hskip \labelsep {\bfseries #1}\hskip \labelsep {\bfseries #2.}]}{\end{trivlist}}

\newenvironment{lemma}[2][Lemma]{\begin{trivlist} \em
\item[\hskip \labelsep {\bfseries #1}\hskip \labelsep {\bfseries #2.}]}{\end{trivlist}}

\newenvironment{definition}[2][Definition]{\begin{trivlist} \em \item[\hskip \labelsep {\bfseries #1}\hskip \labelsep {\bfseries #2.}]}{\end{trivlist}}
\newenvironment{proposition}[2][Proposition]{\begin{trivlist} \em \item[\hskip \labelsep {\bfseries #1}\hskip \labelsep {\bfseries #2.}]}{\end{trivlist}}

\title{Best arm identification in rare events}

\author{
 Anirban Bhattacharjee \\
  School of Technology and Computer Science\\
  Tata Institute of Fundamental Research\\
  Mumbai, India \\
  \texttt{anirban.bhattacharjee\_182@tifr.res.in}
   \And
 Sushant Vijayan \\
  School of Technology and Computer Science\\
  Tata Institute of Fundamental Research\\
  Mumbai, India \\
  \texttt{sushant.vijayanq@tifr.res.in} \\
  \And
 Sandeep K Juneja \\
  School of Technology and Computer Science\\
  Tata Institute of Fundamental Research\\
  Mumbai, India \\
  \texttt{juneja@tifr.res.in} \\
}

\begin{document}
\maketitle
\begin{abstract}
We consider the best arm identification problem in the stochastic
multi-armed bandit framework where 
each arm has a tiny probability of realizing large rewards while
with overwhelming probability the reward is zero. 
A key application of this framework is in online advertising where
click rates of advertisements could be a fraction of a single 
percent
and final conversion to sales, while highly profitable, may again be a small fraction of the click rates. Lately, algorithms for BAI problems have been developed
that minimise sample complexity 
while providing statistical guarantees on the correct arm selection. As we observe, these algorithms can be computationally 
prohibitive. We exploit the fact that the reward process for each arm
is well approximated by a Compound Poisson process
to arrive at algorithms that are faster,
with a small increase in sample complexity.
We analyze the problem in an asymptotic regime as rarity of reward occurrence reduces
to zero, and reward amounts increase to infinity. This  helps illustrate the benefits of the proposed algorithm. 
It also sheds light on the underlying structure of the
optimal BAI algorithms in the rare event setting.
\end{abstract}


\section{Introduction}

Online advertising is ubiquitous in present times, and is used by e-commerce platforms, mobile application developers, marketing professionals etc. Typically, an online advertiser has to decide amongst various product advertisements and choose the one with highest expected reward. Advertisers typically have a period of experimentation where they sequentially show  competing advertisements to the users to arrive at advertisements that elicit 
best response from each customer type (customers maybe clustered based on available information).  
\\
\\
A key feature of online advertising 
is that  while each advertisement maybe shown to a large number
of customers, the click rates on advertisements are usually small.
Typically, these maybe of order one  in a thousand \footnote{https://cxl.com/guides/click-through-rate/benchmarks/}, and a very small percentage \footnote{https://localiq.com/blog/search-advertising-benchmarks/.} of the users who click on an advertisement  end up buying the product (known as the conversion rate). The conversion and click rates can vary significantly depending on the product category. For example, high-end products often have higher click rates but much lower conversion rates compared to standard products. Thus, a key characteristic of the problem is that rarer conversion rates often have very high rewards. 

We study the problem of identifying the best advertisement to show to a customer type
as a best arm identification (BAI) problem in the multi-armed bandit framework.
The rarity of the reward probabilities, and the fact advertisements 
are  shown to a large number of customers, may make the computational effort 
of popular existing adaptive algorithms prohibitive. On the other hand, 
these properties call for sensible aggregation
based algorithms. In this paper, we observe that the rewards from large number 
of pulls from each arm
can be well modelled as a Compound Poisson process, significantly 
simplifying and speeding up the existing {\it optimal} algorithms. 
\\
\\
To illustrate the proposed ideas clearly, we consider a simple stochastic 
BAI problem where
agent is given a set of $K$ unknown probability distributions (arms) 
that can be sampled sequentially.  The agent's objective is to declare the arm with the highest mean with a pre-specified confidence level $1- \delta$, while minimizing the expected number of samples (sampling complexity). In the literature, this is popularly 
known as the fixed-confidence setting, and the algorithms
that provide $1-\delta$ confidence guarantees
are referred to as $\delta$-correct.
\\
\\
Best arm identification problems
are also popular in simulation community
where these are better known as  ranking and selection problems (for example see \cite{goldsman1983ranking,chan2006sequential}).
Classical problem involves many complex simulation models
of practical systems such as supply chain design, traffic network and so
on, and the aim is to identify with high probability, the system
with the highest  expected reward, using minimum computational budget. In many systems, the performance measure of interest may correspond
to a rare event, e.g., a manufacturing plant shut down probability, 
or computer system unavailability fraction. The algorithms that we
propose here are also applicable in optimal computational resource allocation
in simulating such systems. 

\noindent {\bf Related literature:} 
In the learning theory literature, \cite{even2006action} were amongst the first to consider the 
fixed confidence BAI problem. They proposed a  successive elimination algorithm
 (see section F of supplementary material). Upper Confidence Bound (UCB) based algorithms were proposed in \cite{auer2002finite,jamieson2014lil}, wherein the arm with highest confidence index is sampled. These algorithms usually stop when the difference between arm indices breaches a certain threshold (see \cite{jamieson2014best} for more details).  Sample complexity of these algorithms was
 shown to match the lower bound within a constant. Motivated by Bayesian approaches in \cite{russo2016simple}, \cite{jourdan2022top} proposes top-two algorithms that propose a challenger to the current empirical best arm and sample between the challenger and the empirical best arm with a pre-defined probability $\beta$. Although these algorithms are $\beta$-optimal \footnote{see \cite{jourdan2022top} for definition} they are not known to be asymptotically optimal in the sense defined in \cite{kaufmann2016}.
 The sample complexity of these  algorithms is typically analyzed in an asymptotic
 regime where $\delta \rightarrow 0$.  
\cite{kaufmann2016} and \cite{kaufmann2016a} derived a more general lower bound (as a maxmin formulation) on the sample complexity. Based on this lower bound a Track-and-Stop algorithm (TS) was proposed for
arm distributions restricted to single parameter exponential families (SPEF), and was shown to match the lower bound even to a constant  (as $\delta \to 0$).\cite{shubhada2019,agrawal2020optimal} extended the TS algorithms to more general distributions. The optimal TS algorithms in the literature, 
proceed iteratively. At each iteration, the observed empirical parameters are plugged into the lower bound max-min problem to arrive at prescriptive optimal sample allocations to each arm, that then guide the sample allocations. 
As is known, and as we observe, these algorithms are computationally prohibitive, especially since in our rare advertising settings, the informative non-zero reward samples (those instances where users buy products) are rare. This motivates the paper's goal to arrive at computationally efficient algorithms that exploit the Compound Poisson structure of the arm reward process, with a small increase in sample complexity.
\\
\\
\noindent
{\bf Contributions:} 
We develop a rarity framework where the reward success probabilities 
are modelled as a function of $\gamma^{\alpha}$ for arm dependent $\alpha >0$
and  $\gamma$ is $ >0$ and small. The rewards are modelled 
to be of order $\gamma^{-\alpha}$ so that the expected rewards across 
arms are comparable (otherwise, we a-priori know arms with small or large expected rewards). We assume that arm specific upper bounds on rewards are available to us.   In this framework, we propose a computationally efficient $\delta$-correct algorithm that is nearly asymptotically optimal for small $\gamma$. This algorithm (Approximate Track and Stop) is based on existing track and stop algorithms that are simplified through
a Compound Poisson approximation to the bandit reward process. 
The Poisson approximation can be seen to be  tight as $\gamma \rightarrow 0$
and we provide bounds on the deviations due to Poisson approximation.
 Further, we give an asymptotically valid upper bound on the sample complexity illustrating that the increase in sample complexity is marginal compared to the computational benefit.
The rarity structure helps us shed further light on the optimal 
sample allocations across arms in our BAI problem. 
 We identify five different regimes depending on the rarity differences between the arms.
Finally, we compare experimentally with the TS algorithm in \cite{agrawal2020optimal} for bounded random rewards. We find that for realistic rare event probabilities and reward structure, our algorithm is 6-12 times faster than the TS algorithm with a small increase (1-13 \%) in sample complexity.
\\
\\
The rest of the paper is organized as follows: 
 Section \ref{setup} formally introduces the problem, rare event setting and provides some background material. Section \ref{approx_prob} introduces the approximate problem, analyzes its deviations from the exact problem and gives the optimal weight asymptotics, Section \ref{approx_algo_section} outlines the details of the Approximate Track and Stop (TS(A)) algorithm, $\delta$-correctness, sample complexity guarantee and computational benefits of the algorithm. Section \ref{num_experimental_section} presents some experimental results and we conclude in Section 6. The proofs of various results and further technical details are furnished in the supplementary material.

\section{Modelling Framework}\label{setup}

Consider a $K$-armed bandit with each arm's distribution denoted by $p_i$, $i \in [K]$. We denote such a bandit instance by $p$. 
For any distribution $\eta$, let $\mu(\eta)$ denote its mean and 
$\textrm{supp}(\eta)$ denote its support. 
Further,  
let $KL(\eta,\kappa)
=\mathbb{E}_{\eta} \log \left (\frac{d \eta}{d \kappa} \right ) $ denote the Kullback-Leibler divergence between two measures $\eta$ and $\kappa$,
where $E_{\eta}$ denotes the expectation operator under $\eta$.
We  assume that $\textrm{supp}(p_i)$ is finite for each $i$.
Further, this set may not be known to the agent. 
However, there is a lower bound 0 and an upper bound $B_i$ for $\textrm{supp}(p_i)$  and that is known to the agent. 
The agent's goal is to sequentially sample from these arms
using a policy that at any sequential step
$t$, may depend upon all the generated data before time $t$. 
The policy then stops at a random stopping time
and declares an arm that it considers to have the highest mean. 
A sampling strategy, a stopping rule and a recommendation rule are together called a best arm bandit algorithm. A best arm bandit algorithm that correctly recommends the arm with the highest mean with probability at least $1-\delta$ (for a pre-specified $\delta \in (0,1)$) is said to be $\delta$-correct.

This BAI problem has been well studied, and lower bounds on sample complexity under $\delta$-correct algorithms 
have been developed along with algorithms
that match the lower bound 
asymptotically as $\delta \rightarrow 0$. Below, we first state the lower bound in 
Theorem~\ref{lower_bound_theorem},
and then briefly outline an algorithm
that asymptotically matches it. 
The lower bounds were developed by \cite{kaufmann2016}) for single parameter exponential family 
of distributions and were  generalized to bounded and 
heavy-tailed distributions by \cite{agrawal2020optimal}.
Let 
\begin{equation}
\mathcal{K}^{L, B}_{inf}(\eta,x) \coloneqq \min_{\substack{\textrm{supp}(\kappa) \subseteq [0, B] \\ \mu_\kappa \leq x}} KL(\eta,\kappa)
\end{equation}

\begin{equation}
\mathcal{K}^{U, B}_{inf}(\eta,x) \coloneqq \min_{\substack{\textrm{supp}(\kappa) \subseteq [0, B] \\ \mu_\kappa \geq x}} KL(\eta,\kappa).
\end{equation}

Henceforth, we suppress the dependence on $B$ above 
to ease the presentation. This should not cause confusion 
in the following discussion. For brevity, we'll denote $\mu_{p_i}$ by $\mu_i$ for each $i \in [K]$. As is customary
in the BAI literature, we assume that best arm is unique
and without loss of generality, $\mu_1 > \mu_i$ for
$i \in [K] \backslash \{1\}$.

\begin{theorem}{5 in \cite{agrawal2020optimal}}\label{lower_bound_theorem}
For our bandit problem, any $\delta$-correct algorithm with stopping rule $\tau_\delta$, satisfies 
\begin{equation*}
\mathbb{E}[\tau_\delta] \geq \frac{1}{V^{*}(p)} \log\Big(\frac{1}{2.4\delta}\Big),
\end{equation*}
where $V^{*}(p)$ equals
\begin{equation}\label{sample_lower_bound}
\max_{w \in \Sigma_K} \min_{i\neq 1} \inf_{x \in [\mu_i,\mu_1]}w_1\mathcal{K}^{L}_{inf}(p_1,x)+w_i \mathcal{K}^{U}_{inf}(p_i,x),
\end{equation}
$\Sigma_K$ being the $K$-dimensional probability simplex.
\end{theorem}

Optimal track and stop (TS) algorithms in the literature that match the lower
bound asymptotically as $\delta \rightarrow 0$
briefly involve the following features
(see, \cite{kaufmann2016}, \cite {agrawal2020optimal}, \cite{agrawal2021optimal}
    for details and justification of such track and stop algorithms. We also discuss
    existing algorithms further in Section F of supplementary material.)
\begin{enumerate}
    \item 
    Arms are sampled sequentially in batches. 
    At stage
    $t$, each arm is sampled at least
    order $\sqrt{t}$ times (this sub linear 
    exploration ensures that no arm is starved).
    \item
    Empirical distributions $\hat{p}_t$
    are plugged into the lower bound 
    that is solved to determine 
    the prescriptive proportions $\hat{w}_t$.
    \item
    The algorithm then samples to closely track these proportions.
    \item 
    The algorithm stops when the log-likelihood ratio
    at stage $m$ exceeds 
    a threshold $\beta(m,\delta)$ (set close to $\log(1/\delta)$).
At stage $m$, the log likelihood ratio equals
    \[
    \begin{aligned}
     \underset{b \neq k^*}
     {\min}\underset{x \leq y}{\inf} &N_{k^*}(m)\mathcal{K}_{inf}^{L}(\hat{p}_{k^*}(m),x) + N_{b}(m)\mathcal{K}_{inf}^{U}(\hat{p}_b(m),y),
    \end{aligned}
    \]
    where $k^*$ denotes the arm with the largest sample mean, 
    each $N_a(m)$ denotes the samples of arm $a$ amongst $m$ samples.  
\end{enumerate}

As is apparent, the above algorithm involves repeatedly solving 
 the lower bound problem, and this is computationally demanding, particularly when nonzero rewards are rare and occur with very low probabilities. 

\subsection{The Rare Event Setting}\label{rare_event_setup}

We now specialize the BAI setting  to  illustrate our rare event framework where the rewards
from each arm take positive values with small probabilities. Further, while
the expected rewards across arms are of the same order, the realized rewards and the associated
probabilities may be substantially different.
\\
\\
Concretely, suppose that $\gamma$ is a small positive value (say of order $10^{-2}$ or lower) and
corresponding to each arm distribution $p_i$, we have a rarity
index $\alpha_i >0$. The support of arm $i$ takes
values $a_{ij} \gamma^{-\alpha_i}$,  each with probability
$p_{ij} \gamma^{\alpha_i} >0$ for $j \leq n_i < \infty$. Under
each $p_i$, the realized reward takes value zero with probability close 
to 1. To summarize,
\begin{equation*}
\begin{aligned}
&\mathbb{P}_{X \sim p_i}(X=a_{ij}\gamma^{-\alpha_i}) = p_{ij}\gamma^{\alpha_i}, \ j \in [n_i] \\
&\mathbb{P}_{X \sim p_i}(X=0) = 1 - \sum_j p_{ij}\gamma^{\alpha_i}.
\end{aligned}
\end{equation*}
The arm means are given by $\mu_i=\sum_j a_{ij}p_{ij}$ and are independent of $\gamma$.
We further assume 
that an upper bound $B_i\gamma^{-\alpha_i}$ for each arm $i$ 
is known to the agent.

The above rarity framework brings out the benefits  of the proposed approximations cleanly for small $\gamma$ in our theoretical analysis. 
However, in executing the associated algorithm, we don't need to separately
know the values of $\gamma$ and each $\alpha_i$.


\subsection{The Poisson Approximation of KL Divergence}\label{poisson_punchline}
We motivate in this section the approximate form of KL divergence that we shall use. The following well-known result, shown in section A.5 of the supplementary material for completeness, is used to motivate our approximation.
\begin{proposition}{1}\label{poisson_approx_result}
Let $\tau^{(1)}_{ij}$ denote the minimum number of samples of arm $i$ needed to see the reward $a_{ij}\gamma^{-\alpha_i}$, i.e. the first arrival time of the support point $j$. Similarly, let $\tau^{(k)}_{ij}$ be the $k$-th arrival time of support point $j$, 
\\

Let $N_{ij}(t)$ be the number of times the reward $a_{ij}\gamma^{-\alpha_i}$ is returned by arm $i$ in $\lceil t\gamma^{-\alpha_i}\rceil$ trials ($t \in \mathbb{R}$). Then as $\gamma \to 0$,

\begin{enumerate}[label=(\alph*)]
    \item $\mathbb{P}(\tau^{(k)}_{ij} > t\gamma^{-\alpha_i}) \to e^{-p_{ij}t}$,
    \item $N_{ij}(t)$ $\xrightarrow{D}$ $\mathrm{Poisson}$($p_{ij}t$). 
\end{enumerate}
Further for all support points, $\{\mathrm{Poisson}(p_{ij}t)\}_j$ is a collection of mutually independent random variables.
\end{proposition}
This implies that in rare event setting, the distribution of the counting process $N_{ij}(t)$ for each support point $a_{ij}\gamma^{-\alpha_i}$ is well-approximated by a Poisson process. We now argue that when $\gamma$ is small enough, the KL divergence between arm distributions $p_i$ and $\tilde{p}_i$ of same rarity can be approximated by a sum of KL divergences between independent $\mathrm{Poisson}$ variables.\\
\\
 Let $X_{1:m}$ and $\tilde{X}_{1:m}$ be two sets of i.i.d samples of size $m$ from $p_i$ and $\tilde{p}_i$ respectively. The corresponding measures are the product measures $p_i^{\otimes m}$ and $\tilde{p}_i^{\otimes m}$ respectively. By the tensorization property of KL-divergence, we have that
\begin{equation}\label{tensor_KL}
KL\big(p_i^{\otimes m},\tilde{p}_i^{\otimes m}\big)=mKL(p_i,\tilde{p}_i)
\end{equation}
In the following discussion we set $m=\lceil t\gamma^{-\alpha_i}\rceil$. Consider the vector-valued random variable $(N_{ij}(t))_{j \in [n_i]}$ and its counterpart $(\tilde{N}_{ij}(t))_{j \in [n_i]}$ under $\tilde{p}_i$. Note that they are functions of the samples $X_{1:\lceil t\gamma^{-\alpha_i}\rceil},\tilde{X}_{1:\lceil t\gamma^{-\alpha_i}\rceil}$. Since we can also reconstruct a permutation of these samples from  \textbf{$({N}_{ij}(t))_j$},\textbf{$(\tilde{N}_{ij}(t))_j$}, we have that 
\begin{equation*}
\begin{aligned}
KL\big(p_i^{\otimes \lceil t\gamma^{-\alpha_i}\rceil},\tilde{p}_i^{\otimes \lceil t\gamma^{-\alpha_i}\rceil}\big) = KL\big(\nu((N_{ij}(t))_j),\nu((\tilde{N}_{ij}(t))_j)\big)
\end{aligned}
\end{equation*}
where $\nu(A)$ is the measure of a random variable $A$. Now, it can easily be shown from \hyperref[poisson_approx_result]{Proposition 1} that
\begin{equation*}
\begin{aligned}
&KL\big(p_i^{\otimes \lceil t\gamma^{-\alpha_i}\rceil},\tilde{p}_i^{\otimes \lceil t\gamma^{-\alpha_i}\rceil}\big)\\
\approx  &\sum_j KL(\textrm{Poisson}(p_{ij}t),\textrm{Poisson}(\tilde{p}_{ij}t))\\
= &t \bigg[ \sum_j p_{ij} \log \Big(\frac{p_{ij}}{\tilde{p}_{ij}}\Big) + (\tilde{p}_{ij} - p_{ij})\bigg].
\end{aligned}
\end{equation*}
for $\gamma$ small enough. Then, combining the approximation above with the relation (\ref{tensor_KL}) gives
\begin{equation}\label{approx_KL}
KL(p_i,\tilde{p}_i)\approx \gamma^{\alpha_i}\bigg[\sum_j p_{ij} \log \Big(\frac{p_{ij}}{\tilde{p}_{ij}}\Big) + (\tilde{p}_{ij} - p_{ij})\bigg].
\end{equation}
 This approximation is used to motivate the approximate lower bound problem in the next section.

\section{Approximate Lower Bound Problem}\label{approx_prob}

For each $i$, if $B_i \notin \textrm{supp}(p_i)$, let $\tilde{n}_i = n_i+1$ and set $a_{i \tilde{n}_i} = B_i$, else $\tilde{n}_i = n_i$. 
The Poisson approximation of the KL divergence (see section \ref{poisson_punchline}) suggests that in lieu of equation \hyperref[sample_lower_bound]{(3)}, which is computationally expensive to solve, one could consider the following approximate problem when the rarity $\gamma$ is small (the summations over $j$ below correspond to $j \in [\tilde{n}_i]$).
\begin{equation}\label{approx_lower_bound_problem}
\begin{aligned}
V^*_{a}(p) \coloneqq \max_{w \in \Sigma_K} \min_{i\neq 1} \underset{\underset{\sum_ju a_{1j}\tilde{p}_{1j}}{\sum_j a_{ij}\tilde{p}_{ij} \geq}}{\inf}\bigg\{w_1\gamma^{\alpha_1}\bigg[ \sum_j p_{1j} \log \Big(\frac{p_{1j}}{\tilde{p}_{1j}}\Big) + (\tilde{p}_{1j} - p_{1j})\bigg] + w_i\gamma^{\alpha_i}\bigg[ \sum_j p_{ij} \log \Big(\frac{p_{ij}}{\tilde{p}_{ij}}\Big)
 + (\tilde{p}_{ij} - p_{ij})\bigg]\bigg\}.
\end{aligned}
\end{equation}
The minimization in \ref{sample_lower_bound} will now be replaced with the approximation in  \ref{approx_KL}. 
Above, instead of allowing $\tilde{p}_i$ to have the support
$[0, B_i \gamma^{-\alpha_i}]$, we limited its support 
to that of $p_i$ extended 
to allow point  $B_i \gamma^{-\alpha_i}$. 
This is justified in Sections A.1-A.2 of the supplementary material. 

Let 
\begin{equation}\label{P_i_def}
\mathcal{P}_i \coloneqq \underset{x \in [\mu_i,\mu_1]}{\inf}w_1\mathcal{K}^{L}_{inf}(p_1,x)+w_i \mathcal{K}^{U}_{inf}(p_i,x)
\end{equation}
denote the inner minimisation problem
in \ref{sample_lower_bound}
and let
\begin{equation} \label{P_i_def2}
\begin{aligned}
\mathcal{P}_{i,a} \coloneqq \underset{\underset{\sum_j a_{1j}\tilde{p}_{1j}}{\sum_j a_{ij}\tilde{p}_{ij} \geq} }{\inf}w_1\gamma^{\alpha_1}\bigg[ \sum_j p_{1j} \log \Big(\frac{p_{1j}}{\tilde{p}_{1j}}\Big) + (\tilde{p}_{1j} - p_{1j})\bigg] + w_i\gamma^{\alpha_i}\bigg[ \sum_j p_{ij} \log \Big(\frac{p_{ij}}{\tilde{p}_{ij}}\Big) + (\tilde{p}_{ij} - p_{ij})\bigg]
\end{aligned}
\end{equation}
denote its approximation (above, we suppress the  dependence on 
$w_1$ and $w_i$ of $\mathcal{P}_i$ and 
$\mathcal{P}_{i,a}$).

By approximating a reformulated version of $\mathcal{P}_i$ that  uses the dual representations of  $\mathcal{K}^{L}_{inf}$ and $\mathcal{K}^{U}_{inf}$ (following the approach used in \cite{honda2010asymptotically,agrawal2020optimal}),
we can show that 
\begin{equation}\label{approx_min_rform}
\begin{aligned}
\mathcal{P}_{i,a}= w_1\gamma^{\alpha_1} \big[\sum_j p_{1j}\log(1+C^{a}_{1i}a_{1j})-C^{a}_{1i}x^{*}_{i,a}\big] + w_i \gamma^{\alpha_i}\big[\sum_j p_{ij}\log(1-C^{a}_{i}a_{ij})+C^{a}_{i}x^{*}_{i,a}\big].
\end{aligned}
\end{equation}
where the quantities $x^{*}_{i,a},C^a_{1i},C^a_{i}$ (the qualifier 'a' reminds us these are for the approximate problem) are defined by the relations:
\begin{equation}\label{approx_prob_cond}
\begin{aligned}
&C^a_{1i}w_1\gamma^{\alpha_1} =C^a_iw_i\gamma^{\alpha_i},\\
&x^{*}_{i,a}=\sum_j \frac{a_{1j}p_{1j}}{1+a_{1j}C^a_{1i}}, \mbox{ and }\\
& x^{*}_{i,a}=\sum_j \frac{a_{ij}p_{ij}}{1-a_{ij}C^a_{i}}.
\end{aligned}
\end{equation}
Section A.4 of the supplementary material provides the step-by-step reformulation, as well as the results that have been used for it (Sections A.1-A.3 and A.5).
   The advantage of our reformulation is that the quantities $C^a_{1i}$ and $C^a_{i}$ have bounded well-defined limits and using (\ref{approx_prob_cond}), we can eliminate the dependence on $x^*_i$ (whose behaviour is not as easy to analyze when $\gamma \to 0$).\\
\\
The discussion in Section  \ref{poisson_punchline} also suggests that $\mathcal{P}_{i,a} \approx \mathcal{P}_i$ and hence, $V^{*}(p)\approx V^{*}_a(p)$. This is shown in the following theorem:

\begin{theorem}{1}\label{approxMaxMinTheorem}
For each $i \in [K]$ and $w \in \Sigma_K$,
$\mathcal{P}_i$, $\mathcal{P}_{i,a}$ are $\mathcal{O}(\gamma^{\max(\alpha_1, \alpha_i)})$.
Furthermore, 
$\underset{\gamma \to 0}{\lim}\frac{\mathcal{P}_i}{\mathcal{P}_{i,a}}=1.$
In addition,  there exist constants $L_{1i}$ and $L_i$, independent of $w$, such that 
\begin{equation*}    
|\mathcal{P}_i - \mathcal{P}_{i,a}| \leq L_{1i}w_1\gamma^{\min(2\alpha_1,\alpha_1 + \alpha_i)}+L_{i}w_i\gamma^{\min(2\alpha_i,\alpha_i + \alpha_1)}.
\end{equation*}
Furthermore,
\begin{equation*}
\begin{aligned}
|V^{*}(p)-V^{*}_a(p)| \leq \underset{i \neq 1}{\max}\max\big(&L_{1i}\gamma^{\min(2\alpha_1,\alpha_1 + \alpha_i)}, L_{i}\gamma^{\min(2\alpha_i,\alpha_i + \alpha_1)}\big).
\end{aligned}
\end{equation*}
\end{theorem}

The proof involves simplifying $\mathcal{P}_i$, $\mathcal{P}_{i,a}$
through Taylor expansions for small $\gamma$.
It is given in the Sections A.4 and B of the supplementary material.



\subsection{Solving the approximate lower bound}\label{approx_lb_section} 
By definition we have that 
\begin{equation*}
V^{*}_a(p)=\max_{w \in \Sigma_K} \min_{i\neq 1}\mathcal{P}_{i,a}.
\end{equation*}
Further, we note that $\mathcal{P}_{i,a}$ is a concave function of $w$ (infimum of linear function of $w$). Maxmin problems with this specific structure were studied in \cite{glynn2004large} (the caveat being that in our 
$\mathcal{K}_{inf}$ definitions in the underlying KL term, the first argument is fixed while we optimize over the second argument, while in 
\cite{glynn2004large}, these orders are reversed. However, all the steps
carry out identically).
The optimal weights $w^*$ are characterized in the following theorem:
\begin{theorem}{1 in \cite{glynn2004large}}
The optimal $w^{*}$ of the maxmin problem \ref{approx_lower_bound_problem} satisfies:
\begin{equation}\label{KL_ratio_sum_eq}
\sum_{i=2}^{K}\frac{\partial \mathcal{P}_{i,a}(w^*)}{\partial w_1}\bigg/ \frac{\partial \mathcal{P}_{i,a}(w^*)}{\partial w_i}=1,
\end{equation}
and $\forall i \neq j$, $i,j \neq 1$,
\begin{equation}\label{wt_kl_eq}
\mathcal{P}_{i,a}(w^*)=\mathcal{P}_{j,a}(w^*).
\end{equation}
These conditions are also sufficient.
\end{theorem}

We can use the above theorem to find closed form expressions (in terms of $w^*$) for $\mathcal{P}_{i,a}$ and $\frac{\partial \mathcal{P}_{i,a}(w^*)}{\partial w_j}$ using (\ref{approx_min_rform}). As a starting point, we identify certain monotonicities present in (\ref{approx_prob_cond}), (\ref{KL_ratio_sum_eq}) and (\ref{wt_kl_eq}) to ease up the process of root-finding via bisection methods.
\\
\\
The equations defining $C^a_{1i}$ and $C^a_i$ imply that $C^a_{i}$ is a decreasing function of $C^a_{1i}$. Mathematically, the implicit functions $g_i(r)$, defined for all $i \neq 1$ as
 \begin{equation*}
 \sum_j \frac{a_{1j}p_{1j}}{1+g_i(r)a_{1j}}=\sum_j \frac{a_{ij}p_{ij}}{1-ra_{ij}}
 \end{equation*}
 are decreasing in $r$. The domain of $g_i$ is chosen such that the RHS in the above equation is positive and finite.
 \\ The optimality equation (\ref{wt_kl_eq}) implies at the optimal weight $w^*$, each $C^a_{1i}$, $i>2$, is an increasing function of $C^a_{12}$. More formally, the functions $\xi_{i}(s)$, $\forall i>2$, implicitly defined through the equation:
 \begin{equation*}
 \begin{aligned}
    &\sum_j p_{1j} \log(1+g_i(\xi_i)a_{1j}) + 
    \frac{g_i(\xi_i)}{\xi_i}\sum_jp_{ij}\log(1-\xi_ia_{ij})
    = \sum_j p_{1j} \log(1+g_2(s)a_{1j}) + 
    \frac{g_2(s)}{s}\sum_j p_{2j}\log(1-sa_{2j})
 \end{aligned}
 \end{equation*}
 are increasing in $s$. The domain of $\xi_i$ is such that the RHS is well-defined. Finally, as  a function of $C^a_{12}$, the LHS in the optimality equation \ref{KL_ratio_sum_eq} is also increasing. Mathematically this means that the functions , $\forall i \neq 1$, 
 \begin{equation*}
 \begin{aligned}
 h_i(s) \coloneqq \bigg(\sum_j p_{1j}\log(1+\xi_ia_{1j}) - \xi_i \Big[\sum_j \frac{a_{1j}p_{1j}}{1+a_{1j}\xi_i}\Big]\bigg)\bigg(\sum_j p_{ij}\log(1-g_i(\xi_i)a_{ij}) + g_i(\xi_i)\sum_j \Big[\frac{a_{ij}p_{ij}}{1-a_{ij}g_i(\xi_i)}\Big]\bigg)^{-1}
 \end{aligned}
 \end{equation*}
 are increasing in $s$. These monotonicities enable one to solve for optimal weights in (\ref{approx_lower_bound_problem}) through simple bisection methods. This is the source of computational benefit of solving (\ref{approx_lower_bound_problem}) vis-a-vis (\ref{sample_lower_bound}). In (\ref{sample_lower_bound}), one has to solve either convex programs ($\mathcal{P}_i$) or a nonlinear system of four equations 
 to arrive at the solution (see Section C of supplementary material).\\
\\
 This enables us to study the behaviour of $w^*$ as $\gamma \to 0.$ We set up some notation first.
\begin{definition}{1} {\em
Two positive valued functions of $\gamma$, $A(\gamma)$ and $B(\gamma)$, are said to be \textit{asymptotically equivalent} if $0<\underset{\gamma \to 0}{\liminf}\frac{A(\gamma)}{B(\gamma)} \leq \underset{\gamma \to 0}{\limsup}\frac{A(\gamma)}{B(\gamma)}<\infty$. We denote this by $A(\gamma) =\Theta(B(\gamma))$.  }
\end{definition}
Let $\alpha_{\max}={\max}_i\alpha_i$. The quantity $\zeta \coloneqq \sum_{\underset{\alpha_i = \alpha_{max}}{i \neq 1,}} h_i(\xi_{i}(0))$
also plays a role in governing the asymptotic behaviour of $w^*$. 

Theorem \hyperref[asymptotic_w]{(2)} provides insight into the optimal weights
in the lower bound problem 
as $\gamma \rightarrow 0$. We discuss its conclusions further in the nex subsection.

\begin{theorem}{2}\label{asymptotic_w}
 The behaviour of $w^*$ as $\gamma \to 0$ is described by the following five cases:\\
 \\
{Case 1: The best arm is not the rarest, $\alpha_{max} \neq \alpha_1.$}
\begin{equation*}
\begin{aligned}
& w^*_1= \Theta(\gamma^{\frac{\alpha_{max}-\alpha_1}{2}}),\\
& w^*_i= \Theta(\gamma^{\alpha_{max}-\alpha_i}) &&\text{for all $i \neq 1$}.
\end{aligned}
\end{equation*}
{Case 2: The best arm is uniquely the rarest, $\alpha_1 = \alpha_{max} > \alpha_i,  i \neq 1.$}
\begin{equation*}
\begin{aligned}
& w^*_2= \Theta( \gamma^{\frac{\alpha_{max}-\alpha_2}{2}}),\\
& w^*_i= \Theta(\gamma^{\alpha_{max}-\alpha_i}) &&\text{for all $i \neq 2.$}
\end{aligned}
\end{equation*}
{Case 3: The best and second best arm only are the rarest, $\alpha_1 = \alpha_2 = \alpha_{max} > \alpha_i, \ \forall i \neq 1,2.$}
\begin{equation*}
 w^*_i = \Theta(\gamma^{\alpha_{max}-\alpha_i}),\text{ for all $i.$}
\end{equation*}
{Case 4: The best arm is the rarest but not uniquely, $\alpha_1 = \alpha_k = \alpha_{max} \geq \alpha_i, \ i \notin \{1,2,k\}$, $\alpha_{max}>\alpha_2$ and $\zeta > 1$.}
\begin{equation*}
\begin{aligned}
& w^*_2= \Theta(\gamma^{\frac{\alpha_{max}-\alpha_2}{2}}),\\
& w^*_i = \Theta( \gamma^{\alpha_{max}-\alpha_i}) &&\text{for all $i \neq 2.$}
\end{aligned}
\end{equation*}
{Case 5: The best arm is the rarest but not uniquely, $\alpha_1 = \alpha_k = \alpha_{max} \geq \alpha_i, \ i \notin \{1,2,k\}$, $\alpha_{max}>\alpha_2$ and $\zeta \leq 1$.}
\begin{equation*}
\begin{aligned}
& w^*_1= \Theta(\gamma^{\alpha_{max}-\alpha_1}),\\
& w^*_i= \Theta( \gamma^{\alpha_{max}-\alpha_i}) &&\text{for all $i \neq 1.$}
\end{aligned}
\end{equation*}
Further, the asymptotic equivalence can be expressed by limits that are functions of parameters of the bandit problem.
\end{theorem}
\begin{proof}
See section C of supplementary material.
\end{proof}
The theorem gives us insight into the behavior of the optimal weights $w^*$ in equation (\ref{approx_lower_bound_problem}). By the fact that $V^*(p)\approx V^*_a(p)$ (Theorem \hyperref[approxMaxMinTheorem]{1}) the optimal weights of actual maxmin problem also will show the same asymptotic behaviour. It is easy to see that substituting these optimal weights in $V^*(p)$ gives us an overall lower bound on the sample complexity as a scalar multiple of $\gamma^{\alpha_{max}}$.

\subsection{Discussion on Theorem \hyperref[asymptotic_w]{2}}

The following lemma will be useful in the subsequent discussion of Theorem \hyperref[asymptotic_w]{2}. Without loss of generality let arm 2 be the one with the second highest mean. We further assume that $\mu_2> \mu_i$ for $i \geq 3$.
\begin{lemma}{1}\label{tweaked_means_meet}
In the maxmin problem (\ref{sample_lower_bound}), let $x^{*}_{i,e}(w^*)$ denote the minimizer of each $\mathcal{P}_i$ for the optimal weights $w^*$. Then, we have $x^{*}_i(w^*) \in [\mu_2,\mu_1] \,\,\,\, \forall i$.
\end{lemma}
\begin{proof}
We shall show this by contradiction. Suppose $x^{*}_{i,e}(w^*) < \mu_2$. Then, from the optimality conditions of $w^*$ (similar to (\ref{KL_ratio_sum_eq}), (\ref{wt_kl_eq})) we have, $\forall i \neq j$, $i,j \neq 1$:
\begin{equation*}
\begin{aligned}
\underset{\mu'_i \geq \mu'_1}{\inf} w^*_1 KL(\mu_1,\mu'_1)+w^*_i KL(\mu_i,\mu'_i) = \underset{\mu'_j \geq \mu'_1 }{\inf} w^*_1 KL(\mu_1,\mu'_1)+w^*_j KL(\mu_j,\mu'_j).
\end{aligned}
\end{equation*}

But we know that this minimization, for each $i \neq 1$, is attained uniquely  by a bandit instance $p'$ where the rest of the arms, except 1 and $i$, are the same as the original bandit instance in consideration, namely, $p$. Both the arms $i$ and $1$ have means $x^{*}_{i,e}(w^*)$ under $p'$. But the assumed hypothesis then implies that $x^{*}_{i,e}(w^*)=\mu'_1<\mu'_2=\mu_2$. That means $p'$ is also in the set $\{\mu'_2 \geq \mu'_1\}$ and hence
\begin{equation*}
\begin{aligned}
\underset{\mu'_i \geq \mu'_1}{\inf} w^*_1 KL(\mu_1,\mu'_1)+w^*_i KL(\mu_i,\mu'_i) > \underset{\mu'_2 \geq \mu'_1 }{\inf} w^*_1 KL(\mu_1,\mu'_1)+w^*_2 KL(\mu_2,\mu'_2).
\end{aligned}
\end{equation*}
However, this contradicts the necessary optimality conditions for $w^*$. Thus, $x^{*}_{i,e}(w^*) \geq  \mu_2$.
\end{proof}
A similar result can also be shown for the approximate  problem (\ref{approx_lower_bound_problem}) (see Section D of supplementary material).\\
\\
In the rare event setting, the non-zero samples from an arm are the informative samples, but they are quite rare. Any algorithm needs to see non-zero (informative) samples from at least some arms before it decides to stop. By Lemma \hyperref[tweaked_means_meet]{1} we know that all arms, except possibly the best and second best ($i =1,2$), will show deviations in their sample mean under max-min optimality. As the TS algorithm and our algorithm track these weights, it is to be expected that the number of samples for arm $i (\neq 1,2)$ is only as high as it takes to see an $\mathcal{O}(1)$ sample mean, but also sufficiently low as to ensure that the probability of sample mean deviation is high. The optimal weights $w^*_i \simeq \gamma^{\alpha_{max}-\alpha_i}$, $\forall i \neq 1,2$, have this feature. This gives the sample complexity for arm $i (\neq 1,2)$ as $\mathcal{O}(\gamma^{-\alpha_i})$ (since the overall sample complexity is $\mathcal{O}(\gamma^{-\alpha_{\max}}$)). On average, each arm thus sees only $\mathcal{O}(1)$ non-zero samples, with a deviation probability $1-\mathcal{O}(\gamma^{\alpha_i}(\mu_1-\mu_i)^2)$ and $\mathcal{O}(1)$ sample mean.

\section{Track and Stop Algorithm}\label{approx_algo_section}

Our algorithm builds upon the Track and Stop (TS) algorithm proposed in \cite{shubhada2019,kaufmann2016a}. We call it Track and Stop (A), to emphasize
thatwe are solving an approximate problem. The algorithm solves the approximate maxmin problem \ref{approx_lower_bound_problem}, and samples according to the weights obtained. The calculation of the sampling weights happen in batches of size $m$. Let $l$ denote the batch index. Within each batch we ensure that each arm gets at least $\sqrt{lm}$ samples. This is done in the same manner as \cite{shubhada2019}. At the end of $l$-th batch, TS(A) evaluates the maximum likelihood ratio $Z_{k^*}(l)$ for the empirical best arm $k^*(l)$ and decides whether to stop or not. The likelihood ratio is given by:
\begin{equation*}
\begin{aligned}
    Z_{k^*}(l) \coloneqq &\underset{b \neq k^*}{\min}\underset{x \leq y}{\inf} N_{k^*}(lm)\mathcal{K}_{inf}^{L}(\hat{p}_{k^*}(lm),x) + N_{b}(lm)\mathcal{K}_{inf}^{U}(\hat{p}_b(lm),y).
\end{aligned}
\end{equation*}
$\hat{p}(t)$ refers to the empirical bandit instance after $t$ samples. $N_i(t)$ denotes to number of pulls of arm $i$ after $t$ samples. TS(A) stops when $Z_{k^*}(l)>\beta(lm,\delta)$, where $\beta(t,\delta)$ is a stopping threshold defined as
\begin{equation*}
    \beta(t,\delta) \coloneqq \log\bigg(\frac{K-1}{\delta}\bigg)+5\log(t+1)+2.
\end{equation*}

Note that we are computing the maximum likelihood ratio by solving the $\mathcal{K}_{inf}$ problems exactly, and not approximately. Although it is relatively expensive to compute these quantities exactly, such computations occur only once for each $l$. The number of samples $N_i(t)$ for each arm $i$ is influenced by the optimal weights that are obtained as solution to the approximate maxmin problem. The precise algorithmic details of TS(A) are given below. 

\begin{algorithm}[h]\label{main_algo}
   \caption{TS(A) algorithm}
   \begin{algorithmic}
   \STATE Generate $\lfloor \frac{m}{K} \rfloor$ samples for each arm.\\
   \STATE $l\leftarrow1$.
   \STATE Compute the empirical bandit $\hat{p}=(\hat{p})_{i \in [K]}$.
   \STATE $\hat{w}(\hat{p})\leftarrow \text{Compute weights according to (\ref{approx_lower_bound_problem})}$.
   \STATE $k^{*}\leftarrow \underset{i \in [K]}{\arg \max }\hspace{0.2cm} \mathbb{E}[\hat{p}_i]$. 
   \STATE Compute $Z_{k^*}(l)$, $\beta(lm,\delta)$. 
   \WHILE{ $Z_{k^*}(l) \geq \beta(lm,\delta) $}
        \STATE $s_i\leftarrow (\sqrt{(l+1)m}-N_i(lm))^+$.
       \IF {$m \geq \sum_{i}s_i$}
           \STATE Generate $s_i$ many samples for each arm $i$.
           \STATE Generate $(m-\sum_i s_i)^+$ i.i.d. samples from $\hat{w}(\hat{p})$. Let $Count(i)$ be occurrence of $i$ in these samples. 
           \STATE Generate $Count(i)$ samples from each arm $i$.
           \ELSE
           \STATE $\hat{s}^*\leftarrow \underset{\hat{s}, s_i \geq \hat{s}_i \geq 0}{\arg \min }\max_i(s_i-\hat{s}_i)$. 
           \STATE Generate $\hat{s}^*_i$ samples from each arm $i$.
       \ENDIF
           \STATE $l\leftarrow l+1$
           \STATE Update empirical bandit $\hat{p}.$
           \STATE $k^{*}\leftarrow \underset{i \in [K]}{\arg \max }\hspace{0.2cm} \mathbb{E}[\hat{p}_i]$.
           \STATE Update $Z_{k^*}(l)$, $\beta(lm,\delta)$. 
           \STATE $\hat{w}(\hat{p})\leftarrow \text{Compute weights according to (\ref{approx_lower_bound_problem})}$.
   \ENDWHILE
   \RETURN $k^*$.
   \end{algorithmic}
\end{algorithm}
\subsection{$\delta$-correctness and sample complexity of TS(A)}
The following theorem guarantees the $\delta$-correctness and gives asymtptotic sample complexity bound for TS(A):
\begin{theorem}{3.}
 The TS(A) is a $\delta$-correct algorithm with the following asymptotic sample complexity bound:
 \begin{equation}
     \limsup_{\delta \to 0} \frac{\mathbb{E}[\tau_{\delta}]}{\log(1/\delta)}\leq \frac{1}{V_{TS(A)}(p)}
 \end{equation}
 where $V_{TS(A)}(p):=\underset{i \neq 1}{\min} \mathcal{P}_{i}(\hat{w}^*(p))$. $\hat{w}^*(p))$ denotes the optimal weights for the approx lower bound problem $V^*_a(p).$
\end{theorem}
See sections E and F in the supplementary material for a proof of Theorem 3. Note that by definition we have $V^*(p) \leq V_{TS(A)}$ and hence we do suffer some loss in sample complexity vis-a-vis the TS algorithm. However, when $\gamma$ is small, the difference is negligible as $w^*(p) \approx \hat{w}^*(p)$.
\subsection{Computational Benefit of Poisson Approximation}
The computational benefit of TS(A) vis-a-vis the exact algorithm, call it
TS (E),  is in how the approximate and exact lower bound problems are solved.\\
\\
Let us first examine the number of operations required in finding the exact lower bound. In our implementation, we used Brent's method for one-dimensional optimization and the bisection method for root finding. To get a relative error of $\epsilon$ in Brent's method (see Chapter 4 in \cite{brent2013algorithms}) we require $\mathcal{O}\big(\log^2\big(\frac{1}{\epsilon}\big)\big)$ operations. The bisection method takes $\mathcal{O}\big(\log\big(\frac{1}{\epsilon}\big)\big)$ for a relative accuracy of $\epsilon$. Lemma 2 (see Section A of the supplementary material) reduces the process of computing $\mathcal{K}_{inf}^{L}$ and $\mathcal{K}_{inf}^{U}$ to a root-finding procedure, causing said computations to take about $\mathcal{O}\big(\log\big(\frac{1}{\epsilon}\big)\big)$ operations. The inner optimization $\mathcal{P}_i$ is a convex optimization that requires $\mathcal{O}\big(\log^2\big(\frac{1}{\epsilon}\big)\big)$ operations. The outer optimization in (\ref{sample_lower_bound}) can be reduced to solving two sets of simultaneous root finding procedures and hence would take $\mathcal{O}\big(\log^2\big(\frac{1}{\epsilon}\big)\big)$. Thus, the total number of operations to solve the exact lower bound (\ref{sample_lower_bound}) is $\mathcal{O}\big(\log^5\big(\frac{1}{\epsilon}\big)\big)$.
\\
\\
In the approximate problem $C_i, C_{1i}$'s are the unknown variables, whose behaviour we analyze. Using $g_i$ (section \hyperref[approx_lb_section]{3.1}) to write $C_{i}$ as a function of $C_{1i}$ requires about $\mathcal{O}\big(\log\big(\frac{1}{\epsilon}\big)\big)$ operations for each such conversion using the bisection method. Then, each of the $C_{1i}$ $(i \neq 2)$, are written as function of $C_{12}$ through $\xi_i$. This again requires about $\mathcal{O}\big(\log\big(\frac{1}{\epsilon}\big)\big)$ operations for each such conversion. Finally the solution of $C_{12}$ through $h_i$ requires another factor of $\mathcal{O}\big(\log\big(\frac{1}{\epsilon}\big)\big)$. This gives the total required number of operations to be $\mathcal{O}\big(\log^3\big(\frac{1}{\epsilon}\big)\big)$. Thus, we are saving about $\mathcal{O}\big(\log^2\big(\frac{1}{\epsilon}\big)\big)$ by solving the approximate problem vis-a-vis the exact one.


\section{Numerical Experiments}\label{num_experimental_section}
We compare the sample complexity and computational time between TS(A) and  Track \& Stop TS(E) algorithm  proposed in \cite{agrawal2020optimal}. We make the comparison across different arms, $\gamma$ and $\alpha$ structures at a confidence level $\delta = 0.01$. We run each algorithm for $100$ sample paths and their average sample complexity and average computational time are reported in the Table 1 below. The algorithm for both TS(E) and TS(A) proceeds in batches
of size $\gamma^{-\alpha_{\max}}$.
\begin{table}[h]
\centering
  \begin{tabular}{|p{2.2cm}|c|c|c|c|}
    \hline
    \multirow{2}{2.2cm}{\textbf{Experiment: ($\gamma$,$\alpha$)}} & \multicolumn{2}{c|}{\textbf{Samples (m)}} & \multicolumn{2}{c|}{\textbf{Runtime (s)}}\\
    \cline{2-5}
    & \textbf{TS(E)} & \textbf{TS(A)} & \textbf{TS(E)} & \textbf{TS(A)}\\
    \hline
    $\gamma=10^{-3}$, $\alpha=(1,1,1)$ & 0.93 & 0.98 & 619.7  & 51.91  \\ 
    \hline
    $\gamma=10^{-2}$, $\alpha=(1,1.5,2)$ & 1.21 & 1.23 & 97.33  & 6.59  \\ \hline
    $\gamma=10^{-3}$, $\alpha=(1,1,1,1,1)$ & 2.03 & 2.22 & 1860.71  &  290.47 \\ \hline
    $\gamma=10^{-2}$, $\alpha=(2,1.5,2,2.5,1)$ & 14.93 & 16.87 & 152.28  & 23.64 \\ \hline
  \end{tabular}
\caption{Comparison between the TS and TS(A) algorithms. Sample complexity is reported in million (m) samples. The computational runtime is reported in seconds (s).}
\end{table}

The table shows for all experiments TS(A) takes slightly more samples (1-13$\%$) to stop and recommend an arm compared to TS. The computational savings of TS(A) is about $6-12$ times the TS algorithm. These simple experiments underscore the trade-off between sample complexity and computational time.  
\section{Conclusion}
The paper proposes a rarity framework to study the fixed confidence BAI problem relevant to online ad placement. In this framework the positive reward probabilities are tiny while the corresponding rewards are quite large. Consequently, the mean rewards are $\mathcal{O}(1)$.\\
We introduce a Poisson approximation to the standard lower bound problem and use it to motivate an algorithm that is computationally faster than the optimal TS algorithm at the cost of a small increase sample complexity. We also use this approximation to derive asymptotic optimal weights which give insight into the lower bound behaviour in the rare event setting. We observe this trade-off between sample complexity and computational time in our numerical experiments.

\bibliographystyle{apalike}
\bibliography{references}

\appendix
\section{The $\mathcal{K}_{inf}$ problem and related reformulations}
\subsection{Dual form of $\mathcal{K}_{inf}$}
The following well-known Lemma 
gives the dual representations of $\mathcal{K}^{U}_{inf}(.,.)$ and $\mathcal{K}^{L}_{inf}(.,.)$. We follow the approach used in \cite{honda2010asymptotically,agrawal2020optimal}.

\begin{lemma}{2}\label{k_inf_dual}
Consider any discrete distribution $\eta$ with a finite support $\{y_j\}_{j \in[n]}$ and an upper bound $B$. We assume $y_j\geq 0,\forall j$ and $0 < x < B$.

a) The dual representation of $\mathcal{K}^{U}_{inf}(\eta,x)$ is
\begin{equation*}
  \mathcal{K}^{U}_{inf}(\eta,x) = \max_{\lambda_U \in \big[0,\frac{1}{B-x}\big]}\sum_{j=0}^{n}\eta_j\log(1+\lambda_U(x-y_j)).  
\end{equation*}
The optimal $\lambda^*_U$ in the dual maximization above is characterised by:
\begin{equation*}
\begin{cases}
\lambda^*_U=0, & \text{if $x<\mu_{\eta}$,}\\
\lambda^*_U=\frac{1}{B-x}, & \text{if $x > \mu_{\eta}$ and $\sum_{j=0}^{n_i}\eta_j\big(\frac{B-x}{B-y_j}\big) < 1$,}\\
\sum_j \frac{y_j\eta_j}{1+\lambda^{*}_U(x-y_j)} = x, & \text{If $x > \mu_{\eta}$, and $\sum_{j=0}^{n}\eta_j\big(\frac{B-x}{B-y_j}\big) \geq 1$.}
\end{cases}
\end{equation*}
The support of the primal optimizer $\kappa^*$ satisfies $supp(\eta) \subseteq supp(\kappa^*) \subseteq supp(\eta) \cup \{B\}$. The constraint is tight at optimality:
\begin{equation*}
\mu_{\kappa^{*}} = x.
\end{equation*}
Further for $y_j \in supp(\eta)$:
\begin{equation*}
\kappa^{*}(y_j)=\frac{n_j}{1+\lambda^{*}_U(x-y_j)}.
\end{equation*}
\\
\\
b) The dual representation of $\mathcal{K}^{L}_{inf}(\eta,x)$ is
\begin{equation*}
  \mathcal{K}^{L}_{inf}(\eta,x) = \max_{\lambda_L \in \big[0,\frac{1}{x}\big]}\sum_{j=0}^{n}\eta_j\log(1-\lambda_L(x-y_j)).  
\end{equation*}
The optimal $\lambda^*_L$ in the dual maximization above is characterised by:
\begin{equation*}
\begin{cases}
\lambda^*_L=0, & \text{\hspace{1cm}if $x\geq \mu_{\eta}$,}\\
\sum_j \frac{(y_j-x)\eta_j}{1-\lambda^{*}_L(x-y_j)} = 0, & \text{\hspace{1cm}If $x < \mu_{\eta}$.}
\end{cases}
\end{equation*}
The support of the primal optimizer $\kappa^{*}$ satisfies $supp(\eta) = supp(\kappa^{*})$. The constraint is tight at optimality:
\begin{equation*}
\mu_{\kappa^{*}} = x.
\end{equation*}
Further for $y_j \in supp(\eta)$:
\begin{equation*}
\kappa^{*}(y_j)=\frac{n_j}{1-\lambda^{*}_L(x-y_j)}.
\end{equation*}
\end{lemma}
\begin{proof}
See sections \ref{k_inf_u} and \ref{k_inf_l}.
\end{proof}
\subsection{Proof of Lemma \hyperref[k_inf_dual]{2a}}\label{k_inf_u}
Define the set $\mathcal{D}:=\{0\} \cup [b,B]$. Suppose a probability distribution $\eta$ has finite support (say $\{0,y_1,...,y_n\}$ for some $n$) from $\mathcal{D}$. Let $\mathcal{M}^+(\mathcal{D})$ denote the set of positive finite measures on $\mathcal{D}$.
We want to find $\mathcal{K}^{U}_{inf}(\eta,x)$, which is defined as 
\begin{equation*}
\mathcal{K}^{U}_{inf}(\eta,x)=\min_{\substack{\textrm{supp}(\kappa) \subseteq \mathcal{D} \\ \mathbb{E}[\kappa] \geq x}} KL(\eta,\kappa).
\end{equation*}
We shall develop a Lagrangian duality for the above quantity in the space $\mathcal{M}^+(\mathcal{D})$.
The Lagrangian with multiplier $\lambda= (\lambda_1,\lambda_2)$ and $\kappa \in \mathcal{M}^+(\mathcal{D})$ is:
\begin{equation*}
\mathcal{L}(\kappa,\lambda) := KL(\eta,\kappa) + \lambda_1(x - \int_{\mathcal{D}}yd\kappa(y)) + \lambda_2(1 - \int_{\mathcal{D}}d\kappa(y)).
\end{equation*}
Then the dual objective becomes
\begin{equation*}
\mathcal{L}(\lambda) := \inf_{\kappa \in \mathcal{M}^+(\mathcal{D})}\mathcal{L}(\kappa,\lambda).
\end{equation*}
Let us define two quantities useful in the analysis:
\begin{equation*}
h(y,\lambda):= -\lambda_2 - \lambda_1y,
\end{equation*}
\begin{equation*}
   Z(\lambda) := \{y \in \mathcal{D}:h(y,\lambda)=0\}. 
\end{equation*}
We define the set 
\begin{equation*}
 \begin{aligned}
 \mathcal{R}_2 &:= \{\lambda \in \mathbb{R}^2: \lambda_1 \geq 0, \lambda_2 \in \mathbb{R},\lambda\neq 0, \inf_{y \in \mathcal{D}}h(y,\lambda) \geq 0\}\\
 &=\{\lambda \in \mathbb{R}^2: \lambda_1 \geq 0, \lambda_2 \in \mathbb{R}, \lambda\neq 0, -\lambda_2 \geq \lambda_1B \geq 0\}.
 \end{aligned}   
\end{equation*}
The lemma below shows that in maximising the dual objective $\mathcal{L}(\lambda)$, it is enough to restrict ourselves to the set $\mathcal{R}_2$.
\begin{lemma}{A.1.a}

\[\max_{\substack{\lambda_1\geq 0, \\ \lambda_2 \in \mathbb{R}}}\mathcal{L}(\lambda) = \max_{\lambda \in \mathcal{R}_2}\mathcal{L}(\lambda)\]
\end{lemma}

\begin{proof}
Suppose $\lambda \notin \mathcal{R}_2$. Then, there is a $y_0 \in \mathcal{D}$ such that $h(y_0,\lambda)<0$. We know that for any $M>0$, we have a measure $\kappa_M \in \mathcal{M}^+(\mathcal{D})$ such that 
\begin{equation*}
    \kappa_M(y_0) = M, \ \frac{d\kappa_M}{d\eta}(y) = 1,  \forall y \in \textrm{supp}(\eta)\backslash \{y_0\}
\end{equation*}
So, we must have that $\textrm{supp}(\kappa_M) = \{y_0\} \cup \textrm{supp}(\eta)$.
\begin{equation*}
\begin{aligned}
  \mathcal{L}(\kappa_M,\lambda) &= \int_{\mathcal{D}}\log\bigg(\frac{d \eta}{d \kappa_M}(y)\bigg)d \eta(y)+\int_{\mathcal{D}}h(y,\lambda)d\kappa_M(y)+\lambda_1 x+\lambda_2 \\
  &=\eta(y_0)\log\bigg(\frac{\eta(y_0)}{M}\bigg)+M h(y_0,\lambda)+ \int_{\mathrm{supp}(\eta)}h(y,\lambda)d \kappa_M(y)+\lambda_1x+\lambda_2.
  \end{aligned}
\end{equation*}
Now as $M \to \infty$ the first two terms tend to $-\infty$ while the other terms remain bounded and gives the result.
\end{proof}
The next lemma characterises the minimizer $\kappa^*$ in the dual objective $\mathcal{L}(\lambda)$. The support of $\kappa^*$ is contained in 
$\mathrm{supp}(\eta) \cup Z(\lambda)$ and its density wrt $\eta$ (wherever it is well-defined) is $1/h(y,\lambda)$.
\begin{lemma}{A.1.b}
For $\lambda \in \mathcal{R}_2$, $\kappa^{*} \in \mathcal{M}^+(\mathcal{D})$ that minimizes $\mathcal{L}(\kappa,\lambda)$ satisfies $\mathrm{supp}(\eta) \subseteq \mathrm{\kappa^*} \subseteq \mathrm{supp}(\eta) \cup Z(\lambda).$
\\
Also, for $y \in \mathrm{supp}(\eta), h(y,\lambda)>0,$ and 
\[\frac{d \kappa^*}{d \eta}=\frac{1}{-\lambda_1-\lambda_2y}.\]
\end{lemma}
\begin{proof}
Given $\lambda \in \mathcal{R}_2$, the inner optimization problem is strictly convex in $\kappa$. This means that a unique minimizer $\kappa^*$ must exist. This $\kappa^*$ must satisfy for any arbitrary $\kappa_1, \kappa_t:= (1-t)\kappa^* + t\kappa_1$, $\frac{\partial \mathcal{L}(\kappa_t,\lambda)}{\partial t}\bigg|_{t=0}\geq 0$. \\
Let us define $\mathcal{L}(t) :=\mathcal{L}(\kappa_t,\lambda)$ which is
\begin{equation*}
    \int_{\mathrm{supp}(\eta)}\log\bigg(\frac{d \eta}{d \kappa_t}(y)\bigg)d \eta(y) + \int_{\mathcal{D}}h(y,\lambda)d\kappa_t(y)+\lambda_1x+\lambda_2.
\end{equation*}
Then, 
\begin{equation*}
  \frac{d \mathcal{L}(t)}{dt} = \int_{\mathrm{supp}(\eta)}\frac{d \eta}{d \kappa^*}(y)(d\kappa^*(y) - d\kappa_1(y))+ \int_{\mathcal{D}}h(y,\lambda)(d\kappa_1(y)-d\kappa^*(y)).  
\end{equation*}
So,
\begin{equation*}
    \frac{d \mathcal{L}(t)}{dt}\bigg|_{t=0} = -\int_{\mathcal{D} \backslash \mathrm{supp}(\eta)}h(y,\lambda) d\kappa^*(y))
    + \int_{\mathcal{D}\backslash \mathrm{supp}(\eta)}h(y,\lambda)(d\kappa_1(y)).
\end{equation*}
Now, $\lambda \in \mathcal{R}^2$ guarantees that $\mathcal{L}^{'}(0)\geq 0$. This completes our proof.
\end{proof}
\begin{remark}{A.1.1}
  If $y \in Z(\lambda)$, then $y$ can only be $-\frac{\lambda_2}{\lambda_1}$. Therefore, we get that $Z(\lambda) = \big\{-\frac{\lambda_2}{\lambda_1}\big\}$, if $\lambda_1 \geq 0, -\frac{\lambda_2}{\lambda_1} \in \mathcal{D}$ and $Z(\lambda) = \emptyset$, otherwise.
\end{remark}
It now remains to find $\underset{\lambda \in \mathcal{R}_2}{\max}\mathcal{L}(\lambda)$ in order to characterise the Lagrangian dual of $\mathcal{K}^{U}_{inf}(\eta,x)$.


\textbf{\textit{If $Z(\lambda) = \Phi$}}, $\textrm{supp}(\kappa^*) = \textrm{supp}(\eta)$. We can then say from the characterization of $\kappa^*$ that 
\[\mathcal{K}_{inf}^{U}(\eta,x) = \max_{\lambda \in \mathcal{R}_2}\sum_{j=0}^{n}\eta_j\log(-\lambda_2-\lambda_1y_j)\]
The first order conditions tell us that $\sum_j \frac{\eta_j}{\lambda_2-\lambda_1y_j} = 1$ and $\sum_j \frac{y_j\eta_j}{\lambda_2-\lambda_1y_j} = x$. Multiplying the first equation by $-\lambda_2$ and the second by $-\lambda_1$ and then adding the two would give us that $\lambda_2-\lambda_1x = 1$. And $\lambda_2\geq\lambda_1B \Rightarrow 1 + \lambda_1x \geq \lambda_1B \Rightarrow \lambda_1 \in \big[0,\frac{1}{B-x}\big]$. We can therefore conclude that
\[\mathcal{K}^{U}_{inf}(\eta,x) = \max_{\lambda_1 \in \big[0,\frac{1}{B-x}\big]}\sum_{j=0}^{n}\eta_j\log(1+\lambda_1(x-y_j))\]

\textbf{\textit{If $Z(\lambda) \neq \Phi$}}, then $-\frac{\lambda_2}{\lambda_1} \leq B$. But $\lambda \in \mathcal{R}_2$ implies that $-\frac{\lambda_2}{\lambda_1} \geq B$. Hence, $-\frac{\lambda_2}{\lambda_1} = B$. Then, we can say that
\[\mathcal{K}_{inf}^{U}(\eta,x) = \max_{\lambda_1 \geq 0}\sum_{j=0}^{n}\eta_j\log(\lambda_1(B-y_j)).\]
Let $\lambda_U^*$ denote the maximizing $\lambda_1$, $\kappa^*(B)$ denote the mass that $\kappa^*$ puts at $B$. Then, we get from the first order conditions that $\sum_j \frac{\eta_j}{\lambda_U^*(B-y_j)} + \kappa^*(B) = 1$ and $\sum_j \frac{y_j\eta_j}{\lambda_U^*(B-y_j)} + B\kappa^*(B)= x$. Multiplying the first equation by B and adding to the second gives us that $B-x = \frac{1}{\lambda_U^*} \Rightarrow \lambda_U^* = \frac{1}{B-x}$. Therefore, in this case,
\[\mathcal{K}^{U}_{inf}(\eta,x) = \sum_{j=0}^{n}\eta_j\log\bigg(\frac{B-y_j}{B-x}\bigg).\]
Note that this can happen iff $\sum_{j=0}^{n}\eta_j\log\bigg(\frac{B-x}{B-y_j}\bigg) \leq 1$.
\\
\\
Irrespective of whether or not $Z(\lambda) = \Phi$, we can say that
\[\mathcal{K}^{U}_{inf}(\eta,x) = \max_{\lambda_1 \in \big[0,\frac{1}{B-x}\big]}\sum_{j=0}^{n}\eta_j\log(1+\lambda_1(x-y_j))\].
Let us define $p(\lambda_1) := \sum_{j=0}^{n}\eta_j\log(1+\lambda_1(x-y_j))$, $\lambda_1 \in \big[0,\frac{1}{B-x}\big]$. Then, $p^{'}(\lambda_1) = \sum_{j=0}^{n}\frac{\eta_j(x-y_j)}{1+\lambda_1(x-y_j)}$ and $p^{''}(\lambda_1) = -\sum_{j=0}^{n}\frac{\eta_j(x-y_j)^2}{(1+\lambda_1(x-y_j))^2}$. The expression for $p^{''}$ leads us to conclude that $p$ is always concave in $\lambda_1$ and hence, must have a unique maximizer.
\\
\\
If $x\leq \mathbb{E}_\eta$, note that $p^{'}(0) = x - \sum_{j=0}^{n} \eta_j y_j \leq 0$, i.e., $p$ decreases in $\big[0,\frac{1}{B-x}\big]$. Hence, we must have $\mathcal{K}^{U}_{inf}(\eta,x) = \max_{\lambda_1 \in \big[0,\frac{1}{B-x}\big]} p(\lambda_1) = p(0) = 0$. Since the maximizer is $\lambda_U^* = 0$, we know from the definition of $Z(\lambda)$ that $Z(\lambda) = \Phi$, and therefore, $\textrm{supp}(\kappa^*) = \textrm{supp}(\eta)$.
\\
\\
If $x > \mathbb{E}_\eta$, then we have that $p^{'}(0) > 0$, meaning that $p$ is increasing at $\lambda_1 = 0$ and therefore, may take the maximum value at either $\lambda_U^* = \frac{1}{B-x}$ or $\lambda_U^* \in \big(0,\frac{1}{B-x}\big)$. Let us first compute $p^{'}\big(\frac{1}{B-x}\big)$.
\begin{equation*}
\begin{aligned}
&p^{'}\big(\frac{1}{B-x}\big) = \sum_{j=0}^{n} \eta_j \frac{(x-y_j)(B-x)}{(B-y_j)}\\
=&(B-x)\sum_{j=0}^{n}\frac{\eta_jx-\eta_jB+\eta_jB-\eta_jy_j}{B-y_j}\\
=&-(B-x)^2\sum_{j=0}^{n}\frac{\eta_j}{B-y_j} + (B-x)\\
=&(B-x)\bigg[1 - \sum_{j=0}^{n}\eta_j\big(\frac{B-x}{B-y_j}\big)\bigg]
\end{aligned}
\end{equation*}
If $p^{'}\big(\frac{1}{B-x}\big) \leq 0$, then $p$ must reach its maximum in $\big(0,\frac{1}{B-x}\big)$. This happens iff $\sum_{j=0}^{n}\eta_j\big(\frac{B-x}{B-y_j}\big) \geq 1$.
\\
\\
If $p^{'}\big(\frac{1}{B-x}\big) > 0$, then $p$ must reach its maximum at $\frac{1}{B-x}$. This happens iff $\sum_{j=0}^{n}\eta_j\big(\frac{B-x}{B-y_j}\big) < 1$.
\\
\\
\begin{remark}{A.1.2}\label{extra_point_condition}
    For the rare event setup, it is now easy to check that mass will be put at $B_i\gamma^{-\alpha_i}$ in $\mathcal{K}^{U}_{inf}(p_i,x)$ iff $x > F_0(\gamma)$, where $F_0(\gamma) := \frac{B_i}{\big(\sum_{j=1}^{n}\frac{a_{ij}p_{ij}}{B_i-a_{ij}}\big)^{-1}+\gamma^{\alpha_i}}$.
\end{remark}


\subsection{Proof of Lemma \hyperref[k_inf_dual]{2b}}\label{k_inf_l}
We want to find
\begin{equation*}
    \mathcal{K}^{L}_{inf}(\eta,x)=\min_{\substack{\textrm{supp}(\kappa) \subseteq \mathcal{D} \\ \mathbb{E}[\kappa] \leq x}} KL(\eta,\kappa)
\end{equation*}
Just as in section \ref{k_inf_u}, we shall develop a Lagrangian dual for  $\mathcal{K}^{L}_{inf}(\eta,x)$. 
The Lagrangian with multiplier $\lambda= (\lambda_1,\lambda_2)$ is:
\begin{equation*}
  \mathcal{L}(\kappa,\lambda) := KL(\eta,\kappa) - \lambda_1(x - \int_{\mathcal{D}}yd\kappa(y)) - \lambda_2(1 - \int_{\mathcal{D}}d\kappa(y))  
\end{equation*}
Similar to section \ref{k_inf_u}, define the quantities
\begin{equation*}
  \mathcal{L}(\lambda) := \inf_{\kappa \in \mathcal{M}^+(\mathcal{D})}\mathcal{L}(\kappa,\lambda),  
\end{equation*}
\begin{equation*}
    \begin{aligned}
    & h(y,\lambda):= \lambda_2 + \lambda_1y,\\
    &Z(\lambda) := \{y \in \mathcal{D}:h(y,\lambda)=0\}
    \end{aligned}
\end{equation*}
and the set
\begin{equation*}
   \begin{aligned}
   \mathcal{R}_2 &:= \{\lambda \in \mathbb{R}^2: \lambda_1 \geq 0, \lambda_2 \in \mathbb{R},\lambda\neq 0, \inf_{y \in \mathcal{D}}h(y,\lambda) \geq 0\}\\
   & =\{\lambda \in \mathbb{R}^2: \lambda_1 \geq 0, \lambda_2 \geq 0, \lambda\neq 0\}.
   \end{aligned} 
\end{equation*}
As in section \ref{k_inf_u} we have the following lemmas:
\begin{lemma}{A.2.a}
\[\max_{\substack{\lambda_1\geq 0, \\ \lambda_2 \in \mathbb{R}}} \mathcal{L}(\lambda) = \max_{\lambda \in \mathcal{R}_2} \mathcal{L}(\lambda)\]
\end{lemma}

\begin{proof}
Suppose $\lambda \notin \mathcal{R}_2$. Then, there is a $y_0 \in \mathcal{D}$ such that $h(y_0,\lambda)<0$. We know that for any $M>0$, we have a measure $\kappa_M \in \mathcal{M}^+(\mathcal{D})$ such that 
\[\kappa_M(y_0) = M, \ \frac{d\kappa_M}{d\eta}(y) = 1,  \forall y \in \textrm{supp}(\eta)\backslash \{y_0\}\]
So, we must have that $\textrm{supp}(\kappa_M) = \{y_0\} \cup \textrm{supp}(\eta)$.
\begin{equation*}
    \begin{aligned}
    \mathcal{L}(\kappa,\lambda) &= \int_{\mathcal{D}}\log\bigg(\frac{d \eta}{d \kappa_M}(y)\bigg)d \eta(y)
    +\int_{\mathcal{D}}h(y,\lambda)d\kappa_M(y)-\lambda_1 x-\lambda_2\\
    &=\eta(y_0)\log\bigg(\frac{\eta(y_0)}{M}\bigg)+M h(y_0,\lambda)
    + \int_{\mathrm{supp}(\eta)}h(y,\lambda)d \kappa_M(y)-\lambda_1x-\lambda_2
    \end{aligned}
\end{equation*}
Now as $M \to \infty$ the first two terms tend to $-\infty$ while the other terms remain bounded and we obtain the desired result.
\end{proof}
\begin{lemma}{A.2.b}
For $\lambda \in \mathcal{R}_2$, $\kappa^{*} \in \mathcal{M}^+(\mathcal{D})$ that minimizes $\mathcal{L}(\kappa,\lambda)$ satisfies $\mathrm{supp}(\eta) \subseteq \mathrm{\kappa^*} \subseteq \mathrm{supp}(\eta) \cup Z(\lambda).$
\\
Also, for $y \in \mathrm{supp}(\eta), h(y,\lambda)>0,$ and 
\[\frac{d \kappa^*}{d \eta}=\frac{1}{\lambda_1+\lambda_2y}.\]
\end{lemma}
\begin{proof}
Given $\lambda \in \mathcal{R}_2$, the inner optimization problem is strictly convex in $\kappa$. This means that a unique minimizer $\kappa^*$ must exist. This $\kappa^*$ must satisfy for any arbitrary $\kappa_1, \kappa_t:= (1-t)\kappa^* + t\kappa_1$, $\frac{\partial \mathcal{L}(\kappa_t,\lambda)}{\partial t}\bigg|_{t=0}\geq 0$. \\
Let us define $\mathcal{L}(t) :=\mathcal{L}(\kappa_t,\lambda)$ which is
\begin{equation*}
    \mathcal{L}(t) =\int_{\mathrm{supp}(\eta)}\log\bigg(\frac{d \eta}{d \kappa_M}(y)\bigg)d \eta(y) + \int_{\mathcal{D}}h(y,\lambda)d\kappa_t(y)-\lambda_1x-\lambda_2.
\end{equation*}
Then, 
\begin{equation*}
\frac{d \mathcal{L}(t)}{dt} = \int_{\mathrm{supp}(\eta)}\frac{d \eta}{d \kappa^*}(y)(d\kappa^*(y) - d\kappa_1(y))+ \int_{\mathcal{D}}h(y,\lambda)(d\kappa_1(y)-d\kappa^*(y)).    
\end{equation*}
So,
\begin{equation*}
 \frac{d \mathcal{L}(t)}{dt}\bigg|_{t=0} = -\int_{\mathcal{D} \backslash \mathrm{supp}(\eta)}h(y,\lambda) d\kappa^*(y)) + \int_{\mathcal{D}\backslash \mathrm{supp}(\eta)}h(y,\lambda)(d\kappa_1(y)).   
\end{equation*}
Now, $\lambda \in \mathcal{R}^2$ guarantees that $\mathcal{L}^{'}(0)\geq 0$. This completes our proof.
\end{proof}
Note that if $y \in Z(\lambda)$ then $y=-\frac{\lambda_2}{\lambda_1}$ if $-\frac{\lambda_2}{\lambda_1} \in \mathcal{D}.$ But because $\lambda \in \mathcal{R}_2$ we have $-\frac{\lambda_2}{\lambda_1}<0$ and hence $Z(\lambda)=\phi.$ This implies $\mathrm{supp}(\kappa^*)=\mathrm{supp}(\eta)$ with the mean and probability conditions
\begin{equation*}
\begin{aligned}
&1=\sum_j \frac{\eta_j }{(\lambda_2+\lambda_1y_j)}\\
&x=\sum_j \frac{y_j\eta_j}{(\lambda_2+\lambda_1y_j)}
\end{aligned}
\end{equation*}
These imply $1=\lambda_2+\lambda_1 x$. As $\lambda_2 \geq 0$, we have $\lambda_1 \leq \frac{1}{x}.$ Thus, denoting the optinal $\lambda_1$ by $\lambda_L^*$, we get that
\[\mathcal{K}^{L}_{inf}(\eta,x)=\sum \eta_j \log(1-\lambda_L^*(x-y_j))\] with $0 \leq \lambda_L^* \leq 1/x$ and the mean equation
\[x=\sum_j \frac{y_j\eta_j}{(1-\lambda_L^*(x-y_j))}.\]

\subsection{Reformulation of the lower bound}\label{reformulation_LB_proof}
We can now use \hyperref[k_inf_dual]{Lemma 2} to simplify $\mathcal{P}_i$ (see \ref{P_i_def} of the main body) in the \hyperref[rare_event_setup]{rare event setting}. We observe that the objective in $\mathcal{P}_i$ is a smooth and strictly convex function. The optimizer, $x^{*}_{i,e}$, is therefore given by first-order stationarity conditions. Using the dual representation, we can write this as
 \begin{equation*}
 w_1 \lambda^*_{L_{1i}}(x^{*}_{i,e})-w_i\lambda^*_{U_i}(x^{*}_{i,e})=0
 \end{equation*}
 where $\lambda^*_{U_i},\lambda^*_{L_{1i}}$ are as in \hyperref[k_inf_dual]{Lemma 2} and are functions of $x^{*}_{i,e}$. Now let us define quantities that are useful in reformulating $\mathcal{P}$ to a form suitable for further analysis. Define
 \begin{equation*}
 \begin{aligned}
 &K_{1i}:= 1-x^{*}_{i,e}\lambda^*_{L_{1i}}(x^{*}_{i,e}),\\ 
 & C_{1i}:= \lambda^*_{L_{1i}}(x^{*}_{i,e})\gamma^{-\alpha_1},\\
 &K_{i}:= 1+x^{*}_{i,e}\lambda^*_{U_i}(x^{*}_{i,e}),\\
 &C_{i}:= \lambda^*_{U_i}(x^{*}_{i,e})\gamma^{-\alpha_i}.\\
 \end{aligned}
 \end{equation*}
These quantities will turn out to have bounded limits as $\gamma \to 0$. The stationarity condition may now be rewritten as 
\begin{equation}\label{C_1iC_iRelation}
    C_{1i}w_1\gamma^{\alpha_1} = C_iw_i\gamma^{\alpha_i}.
\end{equation}
In the rare event setup, the tightness of the constraint in \hyperref[k_inf_dual]{Lemma 2} gives us that 
\begin{equation}\label{twisted_mean_eq}
\begin{aligned}
    x^{*}_{i,e}&=\sum_{j=1}^{n_1}\frac{a_{1j}p_{1j}}{K_{1i}+C_{1i}a_{1j}} = \sum_{j=1}^{n_i}\frac{a_{ij}p_{ij}}{K_{i}-C_{i}a_{ij}} + B_i\gamma^{-\alpha_i}\bigg[1 - \sum_{j=1}^{n}\frac{p_{ij}}{K_i - C_ia_{ij}}\gamma^{\alpha_i} - \frac{1-\sum_{j=1}^{n}p_{ij}\gamma^{\alpha_i}}{K_i}\bigg].  
\end{aligned}
\end{equation}
Since  the primal optimizer has the same support as the underlying distribution in part (b) of \hyperref[k_inf_dual]{Lemma 2}, we must have
\begin{equation}\label{prob_eq}
    \sum_{j=1}^{n}\frac{p_{1j}}{K_{1i} + C_{1i}a_{1j}}\gamma^{\alpha_1} + \frac{1-\sum_{j=1}^{n}p_{1j}\gamma^{\alpha_1}}{K_{1i}} = 1.
\end{equation}
From their definitions and from the stationarity condition, we have the following relationship between $K_{1i}$ and $K_i$:
\begin{equation}\label{w_K_iEq}
    w_1(1-K_{1i})=w_i(K_i-1).
\end{equation}

Let $\mathcal{P}_i = \underset{x \in [\mu_i,\mu_1]}{\inf}\mathcal{K}_i(w_1,w_i,x)$ (see (\ref{P_i_def}) from the main body). We know from the Envelope Theorem that
\[\frac{d\mathcal{K}_i(w_1,w_i,x)}{dx} = -w_1\lambda_{L_{1i}^*} + w_i\lambda_{U_i^*}.\]
The first order stationarity condition $\frac{d\mathcal{K}_i(w_1,w_i,x)}{dx}=0$ implies that $w_1\lambda_{L_{1i}^*} = w_i\lambda_{U_i^*} \ = \phi_i, (\textrm{say})$. Let us define $x^*_i := \arg\min_{x \in [\mu_i,\mu_1]}\mathcal{K}_i(w_1,w_i,x)$. It is easy to infer from our derivations of the $\mathcal{K}_{inf}^L$ and $\mathcal{K}_{inf}^U$ expressions that
\begin{equation}
\begin{aligned}
    &\mathcal{K}_{inf}^L(p_1,x^*_i) = KL(p_1,\tilde{p}_1^{(i)}) \\
    &\mathcal{K}_{inf}^U(p_i,x^*_i) = KL(p_i,\tilde{p}_i)
\end{aligned}
\end{equation}
where
\begin{equation}
\begin{aligned}
    &\tilde{p}_{1j}^{(i)} = \frac{p_{1j}}{1-\lambda_{L_{1i}^*}(x^*_i-a_{1j}\gamma^{-\alpha_1})}=\frac{p_{1j}}{\big(1-\frac{\phi_i}{w_1}x^*_i)+\frac{\phi_ia_{1j}}{w_1\gamma^{\alpha_1}}}
    \\
    &\tilde{p}_{ij} = \frac{p_{ij}}{1+\lambda_{U_{i}^*}(x^*_i-a_{ij}\gamma^{-\alpha_i})}=\frac{p_{ij}}{\big(1+\frac{\phi_i}{w_i}x^*_i)-\frac{\phi_ia_{ij}}{w_i\gamma^{\alpha_i}}}
\end{aligned}
\end{equation}
We note that $\mathbb{E}_{\tilde{p}_1^{(i)}} = \mathbb{E}_{\tilde{p}_i} = x^*_i$.
\\
\\
We can now express $K_{1i}=1-\frac{\phi_i}{w_1}x^*_i
-i$, $K_i=1+\frac{\phi_i}{w_i}x^*_i$, $C_{1i} = \frac{\phi_i}{w_1\gamma^{\alpha_1}}$, $C_i = \frac{\phi_i}{w_i\gamma^{\alpha_i}}$. The following obvious equations will be helpful.
\begin{equation*}
\begin{aligned}
    &K_{1i} = \frac{1-\sum_{j=1}^{n}p_{1j}\gamma^{\alpha_1}}{1-\sum_{j=1}^{n}\tilde{p}_{1j}^{(i)}\gamma^{\alpha_1}}
    \\
    &K_{i} = \frac{1-\sum_{j=1}^{n}p_{ij}\gamma^{\alpha_i}}{1-\sum_{j=1}^{n}\tilde{p}_{ij}\gamma^{\alpha_i}}\\
    &w_1(1-K_{1i})=w_i(K_i-1)=\phi_ix^*_i
\end{aligned}
\end{equation*}

We also claim that 
\begin{equation}\label{k1i_ki_bound}
\begin{aligned}
   &1-\sum_{j=1}^{n}p_{1j}\gamma^{\alpha_1}\leq K_{1i} \leq 1,\\
   &1 \leq K_i \leq \bigg[\frac{1}{1-\frac{\gamma^{\alpha_1}\mu_1}{\max_ja_{ij}(1-\sum_{j=1}^{n}p_{1j}\gamma^{\alpha_1})}}\bigg].
\end{aligned}
\end{equation}
\\
\\
For the proof of the first claim, we see that $K_{1i} = 1-\lambda_{L_{1i}^*}x\leq 1$ because $0 \leq \lambda_{L_{1i}^*} \leq \frac{1}{x} \Rightarrow 0 \leq \lambda_{L_{1i}^*}x \leq 1$. The lower bound on $K_{1i}$ is trivial.
\\
\\
For the proof of the second claim, we see that $K_i=1+\frac{\phi_i}{w_i}x^* \geq 1$. We also have that $w_i(K_i-1) = \phi_ix^* \leq \frac{\phi_ix^*}{K_{1i}} \leq \frac{\phi_ix^*_i}{1-\sum_{j=1}^{n}p_{1j}\gamma^{\alpha_1}}$. This implies that $K_i -1 \leq \frac{\phi_i}{w_i\gamma^{\alpha_i}}.\frac{\gamma^{\alpha_i \mu_1}}{1-\sum_{j=1}^{n}p_{1j}\gamma^{\alpha_1}} \leq \frac{K_i}{max_{j}a_{ij}}.\frac{\gamma^{\alpha_i \mu_1}}{1-\sum_{j=1}^{n}p_{1j}\gamma^{\alpha_1}}$. As the final step, we can conclude from the above chain of inequalities that $K_i\bigg(1 - \frac{1}{max_{j}a_{ij}}.\frac{\gamma^{\alpha_i \mu_1}}{1-\sum_{j=1}^{n}p_{1j}\gamma^{\alpha_1}}\bigg) \leq 1$
\\
\\
These bounds tell us that $K_{1i},K_i \to 1$ as $\gamma \to 0$. Now, we can write $\mathcal{P}_i$ in terms of $K_{1i},K_{i},C_{1i},C_{i}$ as 
\begin{equation}\label{exact_rform_eq}
\begin{aligned}
   \mathcal{P}_i = &w_1\gamma^{\alpha_1}\bigg[\sum_j p_{1j}\log(K_{1i}+C_{1i}a_{1j}) + \frac{(1-\sum_{j=1}^{n}p_{1j}\gamma^{\alpha_1})}{\gamma^{\alpha_1}}\log(K_{1i})\bigg] \\ + &w_i \gamma^{\alpha_i}\bigg[\sum_j p_{ij}\log(K_i-C_ia_{ij}) + \frac{(1-\sum_{j=1}^{n}p_{ij}\gamma^{\alpha_i})}{\gamma^{\alpha_i}}\log(K_{i})\bigg].
\end{aligned}
\end{equation}
The advantage of re-writing $\mathcal{P}_i$ in terms of $K_{1i},K_{i},C_{1i},C_{i}$ is that these quantities have bounded well-defined limits and using equations (\ref{C_1iC_iRelation}),(\ref{twisted_mean_eq}),(\ref{prob_eq}),(\ref{w_K_iEq}), we can eliminate the dependence on $x^*_i$ (whose behaviour is not as easy to analyze when $\gamma \to 0$). The bounds on $K_{1i} \textrm{ and } K_{i}$ will also help us to define the approximate version $\mathcal{P}_{i,a}$ of $\mathcal{P}_i$ (see \ref{approx_min_rform} of main body).

\subsection{Proof of Proposition 1}
Consider i.i.d. draws of the $i$th arm. Define
\begin{equation*}
\begin{aligned}
    &\tau_{ij}^{(1)} \coloneqq \textrm{the first time } a_{ij}\gamma^{-\alpha_i} \textrm{ is seen in arm } i.\\
    &\tau_{ij}^{(k)} \coloneqq \textrm{the } k \textrm{th inter-arrival time of } a_{ij}\gamma^{-\alpha_i} \textrm{ in arm } i.
\end{aligned}
\end{equation*}
Then, we have that 
\begin{equation*}
    \mathbb{P}(\tau_{ij}^{(1)}>n) = (1-\gamma^{\alpha_i}p_{ij})^n
\end{equation*}
Clearly, the $k$th inter-arrival time is independent of all the previous inter-arrival times. Hence
\begin{equation*}
 \mathbb{P}(\tau_{ij}^{(k)}>n_k) = (1-\gamma^{\alpha_i}p_{ij})^{n_k}
\end{equation*}
Now setting $n_k=t\gamma^{-\alpha_i}$ and taking the limit $\gamma \to 0$ we have
\begin{equation*}
\begin{aligned}
    \lim_{\gamma \to 0} \mathbb{P}(\tau_{ij}^{(k)}>t\gamma^{-\alpha_i}) &= \lim_{\gamma \to 0}(1-\gamma^{\alpha_i}p_{ij})^{t\gamma^{-\alpha_i}} \\
    &=e^{p_{ij} t}
\end{aligned}
\end{equation*}
Now as the inter-arrival times are asymptotically independent exponentially distributed, it follows by the standard argument that $N_{ij}(t)$ is asymptotically distributed as Poisson$(p_{ij}t)$. Note that the same argument could have been repeated while assuming two or more support points as a set. We would then get that the count process for the set are asymptotically distributed as sum of the individual Poisson distributions. From computing the Poisson mgf this implies asymptotic independence of these Poisson variables. We omit the arguments as they are standard.

\section{Proof of Theorem 1}\label{approxMaxMinTheoremProof}
In this section alone, we add the superscript $e$ to $C_i,C_{1i}$ to prevent any confusion, since exact and approximate versions are used simultaneously. Let $C_{1i}^{e}, C_i^{e}, x^*_{i,e}$ denote solutions inner minimization problem $\mathcal{P}_i(w)$, and $C_{1i}^{a}, C_{i}^{a}, x^*_{i,a}$ denote solutions to the approximate inner minimization problem $\mathcal{P}_{i,a}(w)$. We have already established bounds on $K_{1i}$ and $K_i$ in A.4. It is straightforward to see from equation \ref{twisted_mean_eq} of the supplementary material and equations \ref{approx_prob_cond} of the main body, that $0 \leq C_{1i}^{e},C_{1i}^{a} \leq \frac{\sum_j p_{1j}}{\mu_i}$,  $0 \leq C_{i}^{e} \leq \frac{K_i}{B_i},C_{i}^{a} \leq \frac{1}{B_i}.$ Using these bounds, one can easily use the definitions of $mathcal{P}_i$, $\mathcal{P}_{i,a}$ to conclude that $\mathcal{P}_i$, $\mathcal{P}_{i,a} = \mathcal{O}(\min(w_1\gamma^{\alpha_1},w_i\gamma^{\alpha_i}))$. $\underset{\gamma \to 0}{\lim}\frac{\mathcal{P}_i}{\mathcal{P}_{i,a}}=1.$ becomes an immediate conclusion.
\\
\\
To establish the bound on $|\mathcal{P}_i - \mathcal{P}_{i,a}|$, we'll follow three broad steps: showing that the solutions to $\mathcal{P}_i$ also approximately solve $\mathcal{P}_{i,a}$; showing that solutions to $\mathcal{P}_i$ and solutions to $\mathcal{P}_{i,a}$ are close; using the Lipschitz property of $\tilde{\mathcal{K}}_{inf}^{L}$ and $\tilde{\mathcal{K}}_{inf}^{U}$ along with the triangle inequality to connect the bounds derived in the earlier steps and arrive at the proof. $\tilde{\mathcal{K}}_{inf}^{L}$ and $\tilde{\mathcal{K}}_{inf}^{U}$ are defined as follows:
\begin{equation*}
\begin{aligned}
    &\tilde{\mathcal{K}}_{inf}^{L}(z)=\gamma^{\alpha_1}\bigg(\sum_j p_{1j}\log(1+za_{1j}) - z\sum_j\frac{a_{1j}p_{1j}}{1-za_{1j}}\bigg)\\
    &\tilde{\mathcal{K}}_{inf}^{U}(m,z)=\gamma^{\alpha_i}\bigg(\sum_jp_{ij}\log(1-za_{ij})+zm\bigg)
\end{aligned}
\end{equation*}
\\
\textbf{Step 1}: \underline{Solutions to exact problem approximately solve approximate problem}
\\
\\
Bounds on $K_{1i}$ (see \ref{k1i_ki_bound}) imply that given any $\epsilon>0$, we have $\gamma$ small enough that $K_{1i} \geq 1-\epsilon$. Then 

\begin{equation*}
    \log\bigg(\frac{1-\epsilon+C_{1i}^{e}a_{1j}}{1+C_{1i}^{e}a_{1j}}\bigg) \leq \log\bigg(\frac{K_{1i}+C_{1i}^{e}a_{1j}}{1+C_{1i}^{e}a_{1j}}\bigg)\leq 0.
\end{equation*}
By Mean Value Theorem (MVT), we have that
\begin{equation*}
   \log\Big(\frac{1-\epsilon+C_{1i}^{e}a_{1j}}{1+C_{1i}^{e}a_{1j}}\Big) \geq -\frac{\epsilon}{1-\epsilon} 
\end{equation*}
and hence,
\begin{equation*}
    -\frac{\epsilon}{1-\epsilon} \leq \log(K_{1i}+C_{1i}^{e}a_{1j})-\log(1+C_{1i}^{e}a_{1j})\leq 0.
\end{equation*}
Thus, for small enough $\gamma$, $\log(1+C_{1i}^{e}a_{1j}) \approx \log(K_{1i}+C_{1i}^{e}a_{1j})$.\\ 
\\
Using the fact that $K_{1i}=1-C_{1i}^{e}x^*_{i,e}\gamma^{\alpha_1}$, we get
\begin{equation*}
    (1-\gamma^{\alpha_1}\sum_j p_{1j})\frac{\log(K_{1i})}{\gamma^{\alpha_1}}\leq -(1-\epsilon)C_{1i}^{e}x^*_{i,e}
\end{equation*}
when $\gamma^{\alpha_1}\sum_jp_{ij} \leq \epsilon.$ Similarly, we have
\begin{equation*}
    (1-\gamma^{\alpha_1}\sum_j p_{1j})\frac{\log(K_{1i})}{\gamma^{\alpha_1}} \geq \frac{-C_{1i}^{e}x^*_{i,e}}{1-C_{1i}^{e}x^*_{i,e}\gamma^{\alpha_1}} = -C_{1i}^{e}x^*_{i,e}+\frac{-(C_{1i}^{e}x^*_{i,e})^2\gamma^{\alpha_1}}{1-C_{1i}^{e}x^*_{i,e}\gamma^{\alpha_1}}
\end{equation*}

Thus, for $\gamma$ small enough, we have 
$(1-\gamma^{\alpha_1}\sum_j p_{1j})\frac{\log(K_{1i})}{\gamma^{\alpha_1}} \approx -C_{1i}^{e}x^*_{i,e}$. In $\mathcal{K}^{L}_{inf}$ (from Lemma \hyperref[k_inf_dual]{2b}), $\tilde{p}$ has no probability mass on the upper bound $B_i$ and hence 
\begin{equation*}
   x^*_{i,e}=\sum_j\frac{a_{1j}p_{1j}}{1-C_{1i}^{e}a_{1j}}. 
\end{equation*}
 This gives us
\begin{equation*}
|\tilde{\mathcal{K}}_{inf}^{L}(C_{1i}^{e})-\mathcal{K}_{inf}^L(K_{1i},C_{1i}^{e})| \leq 2\gamma^{2\alpha_1}\frac{(\sum_j p_{1j})^2}{1-\sum_j p_{1j}\gamma^{\alpha_1}}
\end{equation*}
\\
\\
Bounds on $K_i$, imply that for any $\epsilon>0$, we can choose $\gamma$ (again independently of $w$) so that $K_{i} \leq 1+
\epsilon.$ Then,
\begin{equation*}
    0 \leq \log(K_i+C_i^{e}a_{ij})-\log(1+C_i^{e}a_{ij}) \leq \log\bigg(\frac{1+\epsilon+C_i^{e}a_{ij}}{1+C_i^{e}a_{ij}}\bigg).
\end{equation*}
Now, from MVT we have
\begin{equation*}
    \log(1+\epsilon+C_i^{e}a_{ij})-\log(1+C_i^{e}a_{ij})\leq \frac{\epsilon}{1+C_i^{e}a_{ij}} \leq \epsilon.
\end{equation*}
Thus, $\log(K_i+C_i^{e}a_{ij}) \approx \log(1+C_i^{e}a_{ij})$ when $\gamma$ is small. From $K_i=1+C_i^{e}x^*_{i,e} \gamma^{\alpha_i}$, we have 
\begin{equation*}
    (1-\epsilon) \frac{C_i^{e}x^*_{i,e}}{1+C_i^{e}x^*_{i,e} \gamma^{\alpha_i}} \leq (1-\gamma^{\alpha_i}\sum_j p_{ij})\frac{\log(K_i)}{\gamma^{\alpha_i}} \leq C_i^{e}x^*_{i,e}
\end{equation*}
when $\gamma^{\alpha_i} \leq \epsilon.$ Thus when $\gamma$ small, $(1-\gamma^{\alpha_i}\sum_j p_{ij})\frac{\log(K_i)}{\gamma^{\alpha_i}} \approx C_i^{e}x^*_{i,e}$.\\
\\
We thus have the following bound:
\begin{equation*}
    |\mathcal{K}_{inf}^{U}(K_{i},C_{i}^{e})-\tilde{\mathcal{K}}_{inf}^{U} (x^*_{i,e},C_{i}^{(e)})| \leq \frac{\frac{\mu_1}{\underset{j}{\max}a_{ij}}\gamma^{2\alpha_i}}{1-\frac{\mu_1}{\underset{j}{\max}a_{ij}}\gamma^{\alpha_i}}
    \bigg(\sum_j p_{ij} + \frac{\mu_1}{\underset{j}{\max}a_{ij}}\bigg)
\end{equation*}
It may be noted that the bound does not depend on $w$, which give uniform bounds independent of $w$.
\\
\\
\textbf{Step 2}: \underline{Solutions to exact problem are close to solutions of approximate problem}

So far, we have shown that the $C_{1i}^{e},C_i^{e}$ and $x^*_{i,e}$ that solve the exact problem are also good solutions for the approximate problem. However, the solution to our new approximate problem will be $C_{1i}^{a}, C_{i}^{a}$  and $x^*_{i,a}$. We'll now show that this set of solutions to the approximate problem indeed approaches the set of solutions to the actual problem at the rate of $\gamma^{\min(2 \alpha_i,\alpha_i+\alpha_1)}$ as $\gamma \to 0$.
\\
\\
We have that
\begin{equation*}
\begin{aligned}
    &x^*_{i,e} = \sum_{j=1}^{n}\frac{a_{1j}p_{1j}}{1-C_{1i}^{e}x^*_{i,e}\gamma^{\alpha_1}+C_{1i}^{e}a_{1j}}, \\
    &x^*_{i,a} = \sum_{j=1}^{n}\frac{a_{1j}p_{1j}}{1+C_{1i}^{a}a_{1j}},
\end{aligned}    
\end{equation*}
Note that the above two statements imply that $C_{1i}^{e}$ and $C_{1i}^{a}$ are bounded above by $\frac{\sum_j p_{1j}}{\mu_i}$. We collect the following established results:
\begin{equation*}
\begin{aligned}
&\frac{C_{1i}^{e}}{C_i^{e}}=\frac{C_{1i}^{a}}{C_{i}^{a}}=\frac{w_i\gamma^{\alpha_i}}{w_1\gamma^{\alpha_1}},\\
&x^*_{i,e}> F_0(\gamma) \Rightarrow C_i^{e} = \frac{1}{B_i-x^*_{i,e}\gamma^{\alpha_i}},\\
&x^*_{i,a}> F_0(0) \Rightarrow C_{i}^{a} = \frac{1}{B_i},\\
&x^*_{i,e}\leq F_0(\gamma) \Rightarrow x^*_{i,e} = \sum_{j=1}^{n}\frac{a_{ij}p_{ij}}{1+C_i^{e}x^*_{i,e}\gamma^{\alpha_i}-C_i^{e}a_{1j}}\\
&x^*_{i,a}\leq F_0(0) \Rightarrow x^*_{i,a} = \sum_{j=1}^{n}\frac{a_{ij}p_{ij}}{1-C_i^{e}a_{1j}}
\end{aligned}
\end{equation*}
where $F_0(\gamma)$ is defined in Remark\hyperref[extra_point_condition]{A.1.2}.
In what follows, we shall  let $b_i=\underset{j}{\min} a_{ij}$.
We shall now establish that, for all $w$, the solution to the exact and approximate inner optimisations are close when $\gamma$ is small. We break the analysis into the following four cases. 
\\
\\
\textbf{\textit{Case 1. $x^*_{i,e} \leq F_0(\gamma),  x^*_{i,a} \leq F_0(0)$.}} 

We have that
\begin{equation*}
\begin{aligned}
    x^*_{i,e} - x^*_{i,a} = &\sum_{j=1}^{n} \frac{a_{1j}p_{1j}(1-K_{1i}+a_{1j}(C_{1i}^{a}-C_{1i}^{e}))}{(1+C_{1i}^{a}a_{1j})(K_{1i}+C_{1i}^{e}a_{1j})} \\
    = &\sum_{j=1}^{n} \frac{a_{ij}p_{ij}(1-K_{i}-a_{ij}(C_{i}^{a}-C_{1i}^{e}))}{(1-C_{i}^{a}a_{ij})(K_{i}-C_{i}^{(e)}a_{1j})}
\end{aligned}
\end{equation*}

Splitting terms from the numerator and using $\frac{C_{1i}^{e}}{C_i^{e}}=\frac{C_{1i}^{a}}{C_{i}^{a}}=\frac{w_i\gamma^{\alpha_i}}{w_1\gamma^{\alpha_1}}$, we get the following:
\[A(1-K_{1i}) + B(1-K_{i})=\tilde{A}(C_{1i}^{e}-C_{1i}^{a}) + \tilde{B}\frac{w_1\gamma^{\alpha_1}}{w_i\gamma^{\alpha_i}}(C_{1i}^{e}-C_{1i}^{a})\]
where

\begin{equation*}
\begin{aligned}
    &A := \sum_{j=1}^{n} \frac{a_{1j}p_{1j}}{(1+C_{1i}^{a}a_{1j})(K_{1i}+C_{1i}^{e}a_{1j})} \\
    &\tilde{A} := \sum_{j=1}^{n} \frac{a_{1j}^2p_{1j}}{(1+C_{1i}^{a}a_{1j})(K_{1i}+C_{1i}^{e}a_{1j})} 
    \geq b_1 A\\
    &B := \sum_{j=1}^{n} \frac{a_{ij}p_{ij}}{(1-C_{i}^{a}a_{ij})(K_{i}-C_{i}^{(e)}a_{1j})} \\
    &\tilde{B} := \sum_{j=1}^{n} \frac{a_{ij^2}p_{ij}}{(1-C_{i}^{a}a_{ij})(K_{i}-C_{i}^{(e)}a_{1j})}\geq b_iB  
\end{aligned}
\end{equation*}

Therefore,

\begin{equation*}
\begin{aligned}
    C_{1i}^{e}-C_{1i}^{a} = \gamma^{\alpha_i}\frac{Aw_i(1-K_{1i})+Bw_i(K_i-1)}{\tilde{A}w_i\gamma^{\alpha_i}+\tilde{B}w_1\gamma^{\alpha_1}}
\end{aligned}
\end{equation*}
Using equation (\ref{w_K_iEq}), we can write that 

\begin{equation*}
\begin{aligned}
    C_{1i}^{e}-C_{1i}^{a} = \bigg(\frac{Aw_i+Bw_1}{\tilde{A}w_i\gamma^{\alpha_i}+\tilde{B}w_1\gamma^{\alpha_1}}\bigg)\gamma^{\alpha_i}(1-K_{1i}).
\end{aligned}
\end{equation*}
 Following this, we can use the lower bounds on $\tilde{A}$, $\tilde{B}$ and $K_{1i}$ to conclude that
\begin{equation*}
\begin{aligned}
    |C_{1i}^{e}-C_{1i}^{a}| \leq \bigg(\frac{\sum_{j}p_{1j}}{\min(b_1, b_i)} \bigg)\gamma^{\min(\alpha_1,\alpha_i)}.
\end{aligned}
\end{equation*}
This also tells us that 
\begin{equation*}
\begin{aligned}
    |x^*_{i,e} - x^*_{i,a}| \leq \mu_1\bigg(\sum_{j=1}^{n}p_{1j}\gamma^{\alpha_1} + \frac{B_1\sum_jp_{1j}}{b_1 \land b_i}\gamma^{\alpha_1 \land \alpha_i}\bigg).
\end{aligned}
\end{equation*}
And using a similar computation, we can also prove that
\begin{equation*}
\begin{aligned}
    |C_{i}^{e}-C_{i}^{a}| \leq \frac{\mu_1 \gamma^{\min{\alpha_1,\alpha_i}}}{min(b_1,b_i)(b_i-\mu_1 \gamma^{\alpha_i})}.
\end{aligned}
\end{equation*}

\textbf{\textit{Case 2. $x^*_{i,e} \geq F_0(\gamma),  x^*_{i,a} \geq F_0(0)$.}}
\\
\\
In this case, we can say that 
\begin{equation*}
\begin{aligned}
    |C_{i}^{(e)}-C_{i}^{a}| = \frac{x^*_{i,e}}{B_i(B_i-x^*_{i,e}\gamma^{\alpha_i})}\gamma^{\alpha_i}
\end{aligned}
\end{equation*}
We also have that 
\begin{equation*}
\begin{aligned}
&x^*_{i,e} = \sum_{j=1}^{n}\frac{a_{1j}p_{1j}}{1+\frac{w_i\gamma^{\alpha_i}}{w_1\gamma^{\alpha_1}}C_{i}^{(e)}(a_{1j}-x^*_{i,e}\gamma^{\alpha_1})}\\
&x^*_{i,a} = \sum_{j=1}^{n}\frac{a_{1j}p_{1j}}{1+\frac{w_i\gamma^{\alpha_I}}{w_1\gamma^{\alpha_1}}C_{i}^{a}a_{1j}}.
\end{aligned}
\end{equation*}
Subtracting the two gives us that
\begin{equation*}
\begin{aligned}
    |x^*_{i,e}-x^*_{i,a}| \leq \sum_{j=1}^{n}\frac{a_{1j}p_{1j}\mu_i}{a_{1j}-\mu_1 \gamma^{\alpha_i}}\gamma^{\alpha_1} + \sum_{j=1}^{n}\frac{a_{1j}^2 p_{1j}\mu_i}{B_i(a_{1j}-\mu_1 \gamma^{\alpha_i})}\gamma^{\alpha_i}.
\end{aligned}
\end{equation*}
The above relation, along with the relation between $|C_{1i}^{e}-C_{1i}^{a}|$ and $|x^*_{i,e}-x^*_{i,a}|$ as outlined under Case I, may be used to prove that
\begin{equation*}
\begin{aligned}
    |C_{1i}^{e}-C_{1i}^{a}| \leq D_i\gamma^{\min(\alpha_1,\alpha_i)}
\end{aligned}
\end{equation*}
where $D_i$ is constant depending on arm $p_i$.
\\
\\
\textbf{\textit{Case 3. $F_0(\gamma) \leq x^*_{i,e}, x^*_{i,a} \leq F_0(0)$.}}
\\
\\
A direct conclusion here would be
\begin{equation*}
\begin{aligned}
    |x^*_{i,e}-x^*_{i,a}| \leq |F_0(0)-F_0(\gamma)| \leq  \frac{B_i}{{1+\gamma^{\alpha_i}\sum_j\frac{a_{ij}p_{ij}}{B_i - a_{ij}}}}\big(\sum_{j=1}^{n}\frac{a_{ij}p_{ij}}{B_i - a_{ij}}\big)^2\gamma^{\alpha_i}
\end{aligned}
\end{equation*}
We have that
\begin{equation*}
\begin{aligned}
    x^*_{i,e} - x^*_{i,a} = \sum_{j=1}^{n} \frac{a_{1j}p_{1j}(1-K_{1i}+a_{1j}(C_{1i}^{a}-C_{1i}^{e}))}{(1+C_{1i}^{a}a_{1j})(K_{1i}+C_{1i}^{e}a_{1j})}
\end{aligned}
\end{equation*}
whence we can conclude that
\begin{equation*}
\begin{aligned}
    &|C_{1i}^{e}-C_{1i}^{a}| \leq \frac{(|x^*_{i,e} - x^*_{i,a}| + C^{(e)}x^*_{i,e}\sum_{j=1}^{n}a_{1j}p_{1j}\gamma^{\alpha_1})}{\frac{b_1\mu_i}{1+B_1C^{(a)}}}\\
    \Rightarrow &|C_{1i}^{e}-C_{1i}^{a}| \leq D_i\gamma^{\min(\alpha_1,\alpha_i)}
\end{aligned}
\end{equation*}
where $D_i$ is again a constant depending on arm $p_i$. Lastly, we can show that 
\begin{equation*}
\begin{aligned}
& |C_{i}^{e}-\frac{1}{B_i}| \leq \frac{(1-b_i/B_i)}{b_i \mu_i}B_i\bigg(\sum_j\frac{a_{ij} p_{ij}}{B_i-a_{ij}}\bigg)^2\gamma^{\alpha_i} \\
&|C_{i}^{a}-\frac{1}{B_i}| \leq \frac{\mu_1}{B_i(B_i-\mu_1\gamma^{\alpha_i})}.\gamma^{\alpha_i}
\end{aligned} 
\end{equation*}
to conclude that
\begin{equation*}
\begin{aligned}
    |C_{i}^{e} - C_{i}^{a}| \leq \frac{(1-b_i/B_i)}{b_i \mu_i}B_i\bigg(\sum_j\frac{a_{ij} p_{ij}}{B_i-a_{ij}}\bigg)^2\gamma^{\alpha_i} + \frac{\mu_1}{B_i(B_i-\mu_1\gamma^{\alpha_i})}.\gamma^{\alpha_i}
\end{aligned}
\end{equation*}
\\
\textbf{\textit{Case 4. $x^*_{i,e} \leq F_0(\gamma) < F_0(0) \leq x^*_{i,a}$.}}
\\
\\
We first show that $1/B_i<C_i^{e}$. Suppose this is false. Then, $C_{i}^{a}=1/B_i \geq C_i^{e}$. From equation (\ref{C_1iC_iRelation}) for fixed $w_1,w_i$ and $\gamma$, we have:
\begin{equation*}
\begin{aligned}
 C_{1i}^{a} \geq C_{1i}^{e} 
\Rightarrow x^*_{i,e} > \sum_j \frac{a_{1j}p_{1j}}{1+C_{1i}^{e} a_{1j}} & >\sum_j \frac{a_{1j}p_{1j}}{1+C_{1i}^{a} a_{1j}} =x^*_{i,a}
\end{aligned}
\end{equation*}
But this contradicts the hypothesis of this case. Hence we must have have:
\[\frac{1}{B_i} < C_i^{e} <\frac{1}{B_i-x^*_{i,e}\gamma^{\alpha_i}}
\]
As $C_{i}^{a}=\frac{1}{B_i}$, from above we have 
\[ 1 < \frac{C_i^{e}}{C_{i}^{a}}=\frac{C_{1i}^{e}}{C_{1i}^{a}} \leq 1+\frac{x^*_{i,e}\gamma^{\alpha_i}}{B_i-x^*_{i,e}\gamma^{\alpha_i}}\]

And we can conclude that 
\begin{equation*}
\begin{aligned}
   &|C_{i}^{a}-C_{1i}^{a}| \leq \frac{\mu_1}{B_i-\mu_1 \gamma^{\alpha_1}}\gamma^{\alpha_i}\\
   &|C_{1i}^{a}-C_{1i}^{e}| \leq \frac{(\sum_j p_{1j}) \mu_1 }{\mu_i (B_i -\mu_1 \gamma^{\alpha_i})}\gamma^{\alpha_i}\\
   &|x^*_{i,a}-x^*_{i,e}| \leq \frac{\mu_i^2 B_i^2}{B_i-\mu_i}\gamma^{\min \{\alpha_1,\alpha_i\}}
\end{aligned}
\end{equation*}
This completes the analysis of the four cases and shows that $C_{1i}^{a},C_i^{a},x^*_{i,a}$ are close to $C_{1i}^{e},C_i^{e},x^*_{i,e}$ when $\gamma$ is small.\\
\\
\textbf{Step 3}: \underline{Connecting solutions to exact problem and solutions to approximate problem}
\\
\\
We concluded in Step 1 that
\begin{equation*}
  |\tilde{\mathcal{K}}_{inf}^{L}(C_{1i}^{e})-\mathcal{K}_{inf}^L(K_{1i},C_{1i}^{e})| \leq 2\gamma^{2\alpha_1}\frac{(\sum_j p_{1j})^2}{1-\sum_j p_{1j}\gamma^{\alpha_1}}  
\end{equation*}
and in Step 2 that $|C_{1i}^{e}-C_{1i}^{a}|$ is related to $|x^*_{i,e}-x^*_{i,a}|$ by the equation 

\begin{equation*}
    |C_{1i}^{e}-C_{1i}^{a}| \leq \frac{|x^*_{i,e}-x^*_{i,a}| + \sum_j a_{1j}p_{1j}C_{1i}^{e}x^*_{i,e}\gamma^{\alpha_1}}{\sum_j \frac{a_{1j}^2p_{1j}}{(1+C_{1i}^{a}a_{1j})(1+C_{1i}^{e}(a_{1j}-x^*_{i,e}\gamma^{\alpha_1}))}} \leq \frac{|x^*_{i,e}-x^*_{i,a}| + \mu_1\sum_jp_{1j}\gamma^{\alpha_1}}{\mu_2\Big(\frac{b_1}{1+ B_1\sum_jp_{1j}/\mu_2}\Big)}
\end{equation*}

We have:
\begin{equation*}
    \frac{d}{dz}\tilde{\mathcal{K}}_{inf}^{L}(z)=\gamma^{\alpha_i}\bigg(\sum_j \frac{a_{1j}p_{1j}}{1+za_{1j}}-\sum_j \frac{a_{1j}p_{1j}}{1-za_{1j}}-z\sum_j\frac{a_{1j}^2p_{1j}}{1-za_{1j}}\bigg)
\end{equation*}
Now, the derivative of $\tilde{\mathcal{K}}_{inf}^{L}$ can easily be bounded above by $\mu_1\gamma^{\alpha_1}$. This leads us to the following conclusion.
\begin{equation*}
    |\tilde{\mathcal{K}}_{inf}^{L}(C_{1i}^{e})-\tilde{\mathcal{K}}_{inf}^{L}(C_{1i}^{a})| \leq \frac{\mu_1^2B_1}{\mu_ib_1} \Bigg[\frac{\frac{\mu_1^3}{\mu_ib_1}\gamma^{\alpha_1}+\mu_1^2(1+\frac{B_1 \lor B_i}{b_1 \land b_i})\frac{1}{(b_i - \mu_1\gamma^{\alpha_i})}\gamma^{\alpha_1 \land \alpha_i}}{\mu_i\Big(\frac{b_i}{1+\frac{\mu_1B_1}{\mu_ib_1}}\Big)}\Bigg]\gamma^{\alpha_1} = \mathcal{O}(\gamma^{(2\alpha_1)\land(\alpha_1 + \alpha_i)})
\end{equation*}

where we have used the inequalities $C^{e}_{1i},C^{a}_{1i}\leq \frac{\sum_j p_{1j}}{\mu_i}$ and $b_1\sum_j p_{1j}\leq \mu_1$.\\
\\
We thus have,
\begin{equation*}
\begin{split}
    |\mathcal{K}_{inf}^{L}(K_{1i},C_{1i}^{e})-\tilde{\mathcal{K}}_{inf}^{L}(C_{1i}^{a})| \leq |\mathcal{K}_{inf}^{L}(K_{1i},C_{1i}^{e})-\tilde{\mathcal{K}}_{inf}^{L}(C_{1i}^{e})| +  |\tilde{\mathcal{K}}_{inf}^{L}(C_{1i}^{e})-\tilde{\mathcal{K}}_{inf}^{L}(C_{1i}^{a})| \leq L_{1i} \gamma^{(2\alpha_1)\land(\alpha_1 + \alpha_i)}
\end{split}
\end{equation*}
where $L_{1i}$ is a computable constant, and $L_{1i} \gamma^{(2\alpha_1)\land(\alpha_1 + \alpha_i)}$ can be computed by adding the bounds on $|\mathcal{K}_{inf}^{L}(K_{1i},C_{1i}^{e})-\tilde{\mathcal{K}}_{inf}^{L}(C_{1i}^{e})|$ and $|\tilde{\mathcal{K}}_{inf}^{L}(C_{1i}^{e})-\tilde{\mathcal{K}}_{inf}^{L}(C_{1i}^{a})|$.
\\
\\
Similarly from Step 1 we have:
\begin{equation*}
\begin{aligned}
&|\mathcal{K}_{inf}^{U}(K_{i},C_{i}^{e})-\tilde{\mathcal{K}}_{inf}^{U} (x^*_{i,e},C_{i}^{e})| \leq \frac{\frac{\mu_1}{\underset{j}{\max}a_{ij}}\gamma^{2\alpha_i}}{1-\frac{\mu_1}{\underset{j}{\max}a_{ij}}\gamma^{\alpha_i}}
    \bigg(\sum_j p_{ij} + \frac{\mu_1}{\underset{j}{\max}a_{ij}}\bigg)
\end{aligned}
\end{equation*}
To upper bound $|\mathcal{K}_{inf}^{U}(K_{i},C_{i}^{e})-\tilde{\mathcal{K}}_{inf}^{U}(x^*_{i,a},C_{i}^{a})|$, we can follow a procedure similar to how $|\mathcal{K}_{inf}^{L}(K_{1i},C_{1i}^{e})-\tilde{\mathcal{K}}_{inf}^{L}(C_{1i}^{a})|$ was bounded. We first use the triangle inequality to make the following split.
\begin{equation*}
\begin{aligned}
    |\mathcal{K}_{inf}^{U}(K_{i},C_{i}^{e})-\tilde{\mathcal{K}}_{inf}^{U}(x^*_{i,a},C_{i}^{a})|
    &\leq |\mathcal{K}_{inf}^{U}(K_{i},C_{i}^{e})-\tilde{\mathcal{K}}_{inf}^{U}(x^*_{i,e},C_{i}^{e})|
    + |\tilde{\mathcal{K}}_{inf}^{U}(x^*_{i,e},C_{i}^{e}) - \tilde{\mathcal{K}}_{inf}^{U}(x^*_{i,e},C_{i}^{a})| \\
    &+ |\tilde{\mathcal{K}}_{inf}^{U}(x^*_{i,e},C_{i}^{a}) - \tilde{\mathcal{K}}_{inf}^{U}(x^*_{i,a},C_{i}^{a})|
\end{aligned}
\end{equation*}
In the right hand side of the above inequality, the bound to the first summand was already obtained. The second and third summands can be bounded above by showing that $\tilde{\mathcal{K}}_{inf}^{U}$ is Lipschitz in both its arguments, the Lipschitz constants being computable ones. Thus, we have
\begin{equation*}
\begin{aligned}
|\tilde{\mathcal{K}}_{inf}^{U}(x^*_{i,e},C_{i}^{e}) - \tilde{\mathcal{K}}_{inf}^{U}(x^*_{i,e},C_{i}^{a})| &\leq
\gamma^{\alpha_i}(\mu_1-\mu_i) |C_i^{e}-C_{i}^{a}|\\
&\leq\frac{\mu_1(\mu_1-\mu_2)}{(b_1 \land b_i)(b_i-\mu_1\gamma^{\alpha_i})}\gamma^{(\alpha_1+\alpha_i)\land(2\alpha_i)} \\
&+ \frac{(B_i-b_i)(\mu_1-\mu_2)}{b_i\mu_i}\bigg(\sum_{j=1}^{n} \frac{a_{1j}p_{1j}}{B_i - a_{ij}}\bigg)^2 \gamma^{2\alpha_i}.
.\end{aligned}
\end{equation*}
The bound in the first step was derived by bounding the partial derivative wrt $z$ of $\tilde{\mathcal{K}}_{inf}^{U}(m,z)$. Similarly bounding the partial derivative wrt $m$ gives
\begin{equation*}
\begin{aligned}
 |\tilde{\mathcal{K}}_{inf}^{U}(x^*_{i,e},C_{i}^{a}) - \tilde{\mathcal{K}}_{inf}^{U}(x^*_{i,a},C_{i}^{a})| \leq \gamma^{\alpha_i}\frac{|x^*_{i,e}-x^*_{i,a}|}{b_i}
\end{aligned}
\end{equation*}
$|x^*_{i,e}-x^*_{i,a}|$ is bounded above by the maximum of the upper bounds derived in the four cases of Step 2.
We can therefore conclude that, 
\begin{equation*}
    |\mathcal{K}_{inf}^{U}(K_{i},C_{i}^{e})-\tilde{\mathcal{K}}_{inf}^{U}(x^*_{i,a},C_{i}^{a})| \leq L_i \gamma^{(\alpha_1+\alpha_i)\land(2\alpha_i)}
\end{equation*}
where $L_i$ can be computed as described above. The upper bounds on $|\mathcal{K}_{inf}^{L}(K_{1i},C_{1i}^{e})-\tilde{\mathcal{K}}_{inf}^{L}(C_{1i}^{a})|$ and $|\mathcal{K}_{inf}^{U}(K_{i},C_{i}^{e})-\tilde{\mathcal{K}}_{inf}^{U}(x^*_{i,a},C_{i}^{a})|$ give us the proof of Theorem \ref{approxMaxMinTheorem}. The upper bound on $|V^{*}(p)-V^{*}_a(p)|$ can be inferred immediately.


\section{Proof of Theorem 2}

The proof goes through the following steps: first we analyse the behavior of equation (\ref{wt_kl_eq}) and derive some constraints it imposes on the asymptotic behavior of $C_{1i}^{a},C_{i}^{a}$; utilising this, we then analyse the behaviour of equation (\ref{KL_ratio_sum_eq}) and finally get the five asymptotic regimes noted in the Theorem.
\\
\textbf{Step 1}: \underline{Constraint imposed by equation (\ref{wt_kl_eq}) in the asymptotic behaviours of $C_{1i}^{a},C_{i}^{a}$.}
\\
We first observe that $C_{1i}^{a} \to 0, C_i^{a} \to 0$ as $\gamma \to 0$ cannot happen for any $i \in [K] \backslash \{1\}$, because then equation \ref{approx_prob_cond} would imply that $\mu_1=\sum_{j=1}^{n}a_{1j}p_{1j} = \sum_{j=1}^{n}a_{ij}p_{ij}=\mu_i$. \\
\\
Equation (\ref{wt_kl_eq}) from the main body can be re-written (using envelope theorem) as 
\begin{equation*}
\begin{aligned}
 &w_1 \gamma^{\alpha_1}\bigg(\sum_j p_{1j} \log(1+C_{1i}^{a} a_{1j})- C_{1i}^{a}x^*_{i,a} \bigg)+w_i \gamma^{\alpha_i}\bigg(\sum_j p_{ij} \log(1-C_{i}^{a} a_{ij})+C_{i}^{a}x^*_{i,a} \bigg)\\
 =&w_1 \gamma^{\alpha_1}\bigg(\sum_j p_{1j} \log(1+C_{1k}^{a} a_{1j})+ C_{1i}^{a}x^*_{k,a} \bigg)+w_k \gamma^{\alpha_i}\bigg(\sum_j p_{kj} \log(1-C_{i}^{a} a_{kj})-C_{i}^{a}x^*_{k,a} \bigg)
 \end{aligned}
\end{equation*}
for all $i \neq k$, $i,k \neq 1$. Using equation $w_1C_{1i}^{a}\gamma^{\alpha_1}=w_iC_{i}^{a}\gamma^{\alpha_i}$, we can simplify this equation to
\begin{equation}\label{wKLratio}
 \frac{\sum_j p_{1j} \log(1+C_{1i}^{a} a_{1j})+ \frac{C_{1i}^{a}}{C_{i}^{a}}\sum_j p_{ij} \log(1-C_{i}^{a} a_{ij}) }{\sum_j p_{1j} \log(1+C_{1k}^{a} a_{1j})+ \frac{C_{1k}^{a}}{C_{k}^{a}}\sum_j p_{kj} \log(1-C_{k}^{a} a_{kj})} =1
\end{equation}
for all $i \neq k$. We also re-write (\ref{approx_prob_cond}) from the main body as 
\begin{equation}\label{approx_twisted_mean_eq}
    \sum_j \frac{a_{1j}p_{1j}}{1+C_{1i}^{a}a_{1j}}=\sum_j\frac{a_{ij}p_{ij}}{1-C_{i}^{a}a_{ij}}.
\end{equation}
Now, we analyze the asymptotic behavior of equation (\ref{wKLratio}) as $\gamma \to 0$ on a case-by-case basis.
\\
\\
\textbf{Case 1:} $C_{1i}^{a} \to A_{1}^{a}(> 0), C_i \to 0; C_{1k}^{a} \to A_{1k}^{a}(> 0), C_k^{a} \to 0.$\\
\\
Taking the limit in equation (\ref{wKLratio}) we get
\begin{equation*}
\begin{aligned}
1=&\lim_{\gamma \to 0} \frac{\sum_j p_{1j} \log(1+C_{1i}^{a} a_{1j})+ \frac{C_{1i}^{a}}{C_{i}^{a}}\sum_j p_{ij} \log(1-C_{i}^{a} a_{ij}) }{\sum_j p_{1j} \log(1+C_{1k}^{a} a_{1j})+ \frac{C_{1k}^{a}}{C_{k}^{a}}\sum_j p_{kj} \log(1-C_{k}^{a} a_{kj})}\\
=&\frac{\sum_j p_{1j} \log(1+A_{1i}^{a} a_{1j})- A_{1i}^{a}\sum_j a_{ij}p_{ij}}{\sum_j p_{1j} \log(1+A_{1k}^{a} a_{1j})- A_{1k}^{a}\sum_j a_{kj}p_{kj}} \\   
\end{aligned}
\end{equation*}
Taking $\gamma \to 0$ in (\ref{twisted_mean_eq}), we have that 
\begin{equation*}
\begin{aligned}
  &\sum_j\frac{a_{1j}p_{1j}}{1+A_{1i}a_{1j}} = \sum_j a_{ij}p_{ij}\\
  &\sum_j\frac{a_{1j}p_{1j}}{1+A_{1k}a_{1j}} = \sum_j a_{kj}p_{kj} 
\end{aligned}
\end{equation*}
 Hence,
 \begin{equation*}
   \frac{\sum_j f_j(A_{1i})}{\sum_j f_j(A_{1k})} = 1  
 \end{equation*}
where $f_j(x) := p_{1j}[\log(1+a_{1j}x) - \frac{xa_{1j}}{1+xa_{1j}}]$.
It is easy to check that $f$ is a monotonically increasing function, and therefore the above equation must imply $A_{1i}=A_{1k}$. But this also means that $\mu_i = \mu_k$, which is against our assumption of all means being distinct.
\\
\\
\textbf{Case 2:} $C_{1i}^{a} \to A_{1i}(> 0), C_i^{a} \to 0, C_{1k}^{a} \to 0, C_k^{a} \to A_k(> 0)$
\\
\\
As in Case 1 we take the asymptotic limit on \ref{wKLratio} to get
\begin{equation*}
\begin{aligned}
1=&\lim_{\gamma \to 0} \frac{\sum_j p_{1j} \log(1+C_{1i}^{a} a_{1j})+ \frac{C_{1i}^{a}}{C_{i}^{a}}\sum_j p_{ij} \log(1-C_{i}^{a} a_{ij}) }{\sum_j p_{1j} \log(1+C_{1k}^{a} a_{1j})+ \frac{C_{1k}^{a}}{C_{k}^{a}}\sum_j p_{kj} \log(1-C_{k}^{a} a_{kj})}\\
=&\lim_{\gamma \to 0} \frac{\sum_j p_{1j} \log(1+A_{1i}^{a} a_{1j})- A_{1i}^{a}\sum_j a_{ij}p_{ij}}{\sum_j p_{1j} \log(1+C_{1k}^{a} a_{1j})- \frac{C_{1k}^{a}}{A_{k}}\sum_j p_{kj}\log(1-A_{k}^{a} a_{kj})} \\   
\end{aligned} 
\end{equation*}
which is impossible, because the denominator of the right hand side approaches $0$ as $\gamma \to 0$.
\\
\\
\textbf{Case 3:} $C_{1i}^{a} \to A_{1i}(> 0), C_i^{a} \to A_i(> 0), C_{1k}^{a} \to 0, C_k^{a} \to A_k(> 0)$
\\
\\
We have that
\begin{equation*}
\begin{aligned}
1=&\lim_{\gamma \to 0} \frac{\sum_j p_{1j} \log(1+C_{1i}^{a} a_{1j})+ \frac{C_{1i}^{a}}{C_{i}^{a}}\sum_j p_{ij} \log(1-C_{i}^{a} a_{ij}) }{\sum_j p_{1j} \log(1+C_{1k}^{a} a_{1j})+ \frac{C_{1k}^{a}}{C_{k}^{a}}\sum_j p_{kj} \log(1-C_{k}^{a} a_{kj})}\\
=&\lim_{\gamma \to 0} \frac{\sum_j p_{1j} \log(1+A_{1i}^{a} a_{1j})+ \frac{A_{1i}^{a}}{A_i^{a}}\sum_j p_{ij}\log(1-A_i^{a}a_{ij})}{\sum_j p_{1j} \log(1+C_{1k}^{a} a_{1j})- \frac{C_{1k}^{a}}{A_{k}}\sum_j p_{kj}\log(1-A_{k}^{a} a_{kj})} \\
\end{aligned}
\end{equation*}
which is impossible, because the denominator of the left hand side approaches $0$ as $\gamma \to 0$. That only leaves us with only the following three possibilities.
\\
\\
\textbf{Case 4:} $C_{1i}^{a} \to A_{1i}(\neq 0), C_i^{a} \to A_i(\neq 0), C_{1k}^{a} \to A_{1k}(\neq 0), C_k^{a} \to A_k(\neq 0)$
\\
\\
From \ref{wKLratio}, we know
\begin{equation*}
    \lim_{\gamma \to 0}\frac{\sum_j p_{1j}\log(1 + C_{1i}^{a}a_{1j}) + \frac{w_i\gamma^{\alpha_i}}{w_1\gamma^{\alpha_1}}\sum_j p_{ij}\log(1 - C_i^{a}a_{ij})}{\sum_j{p_{1j}\log(1 + C_{1k}^{a}a_{1j}) + \frac{w_k\gamma^{\alpha_k}}{w_1\gamma^{\alpha_1}}\sum_jp_{kj}\log(1 - C_k^{a}a_{kj})}}
\end{equation*}
which cannot be ruled out as an impossibility.
\\
\\
\textbf{Case 5:} $C_{1i}^{a} \to 0, C_i{a} \to A_i(\neq 0), C_{1k}^{a} \to 0, C_k^{a} \to A_k(\neq 0)$
\\
\\
Using $C_{1i}^{a}w_1 \gamma^{\alpha_1} = C_i^{a} w_i \gamma^{\alpha_i} = \lambda_i \ \forall i \neq 1$  on \ref{wKLratio} gives us that
\begin{equation*}
\begin{aligned}
&\lim_{\gamma \to 0} \frac{C_{1i}^{a}}{C_{1k}^{a}}\frac{\sum_jp_{1j}\frac{\log(1 + C_{1i}^{a}a_{1j})}{C_{1i}^{a}} + \sum_jp_{ij}\frac{\log(1 - C_i{a}a_{ij})}{C_{i}^{a}}}{\sum_jp_{1j}\frac{\log(1 + C_{1k}^{a}a_{1j})}{C_{1k}^{a}} + \sum_jp_{kj}\frac{\log(1 - C_k^{a}a_{kj})}{C_k^{a}}} \\
= &\lim_{\gamma \to 0} \frac{C_{1i}^{a}}{C_{1k}^{a}} \bigg(\frac{\sum_{j}a_{1j}p_{1j} + \sum_j\frac{p_ij}{A_i}\log(1 - A_ia_{ij})}{\sum_{j}a_{1j}p_{1j} + \sum_j\frac{p_{kj}}{A_k}\log(1 - A_ka_{kj})}\bigg) = 1\\
\Rightarrow &\lim_{\gamma \to 0} \frac{C_{1i}^{a}}{C_{1k}^{a}} = \frac{\sum_ja_{1j}p_{1j} + \sum_j\frac{p_{kj}}{A_k}\log(1 - A_ka_{kj})}{\sum_ja_{1j}p_{1j} + \frac{p_{ij}}{A_i}\log(1 - A_ia_{ij})}\\
\Rightarrow &\lim_{\gamma \to 0} \frac{C_{i}^{a}w_i\gamma^{\alpha_i}}{C_{k}^{a}w_k\gamma^{\alpha_k}} = \bigg(\frac{\sum_ja_{1j}p_{1j} + \sum_j\frac{p_{kj}}{A_k}\log(1 - A_ka_{kj})}{\sum_ja_{1j}p_{1j} + \sum_j\frac{p_{ij}}{A_i}\log(1 - A_ia_{ij})}\bigg)
\end{aligned}
\end{equation*}
\\
\\
\textbf{Case 6:} $C_{1i}^{a} \to A_{1i}(\neq 0), C_i^a \to 0, C_{1k}^{a} \to A_{1k}(\neq 0), C_k^{a} \to A_k(\neq 0)$
\\
\\
Using $C_{1i}^{a}w_1 \gamma^{\alpha_1} = C_i^{a} w_i \gamma^{\alpha_i} = \lambda_i \ \forall i \neq 1$  on \ref{wKLratio} gives us that
\begin{equation*}
\begin{aligned}
&\lim_{\gamma \to 0} \frac{C_{1i}^{a}}{C_{1k}^{a}}\frac{\sum_j p_{1j}\frac{\log(1 + C_{1i}^{a}a_{1j})}{C_{1s}^{a}} + \sum_j p_{ij}\frac{\log(1 - C_i^{a}a_{ij})}{C_{i}^{a}}}{\sum_j p_{1j}\frac{\log(1 + C_{1k}^{a}a_{1j})}{C_{1k}^{a}} + \sum_j p_{kj}\frac{\log(1 - C_k^{a}a_{kj})}{C_k^{a}}}\\
= &\frac{\sum_j p_{1j}\log(1+A_{1i}a_{1j}) -A_{1i}\mu_i}{\sum_j p_{1j}\log(1 + A_{1k}a_{1j}) + \frac{A_{1k}}{A_k}\sum_j p_{kj}\log(1 - A_ka_{kj})} = 1
\end{aligned}
\end{equation*}

\textbf{Step 2}: \underline{Analysis of equation \ref{KL_ratio_sum_eq} of the main body.}
\\
\\
The Envelope Theorem guarantees that equation \ref{KL_ratio_sum_eq} of the main body can be rewritten as 
\begin{equation}\label{KL_ratio_sum_eq_raw}
    \sum_{i=2}^{K}\frac{KL(p_1,\tilde{p}_1^{(i)})}{KL(p_i,\tilde{p}_i)} = \sum_{i=2}^{K}\frac{\gamma^{\alpha_1}(\sum_jp_{1j}\log(1 + C_{1i}^{a}a_{1j}) - C_{1i}^{a}\sum_ja_{1j}\tilde{p}_{1j}^{(i)})}{\gamma^{\alpha_i}(\sum_jp_{ij}\log(1 - C_{i}^{a}a_{ij}) + C_i^{a}\sum_j a_{ij}\tilde{p}_{ij})} = 1
\end{equation}

because $\frac{\partial \mathcal{P}_{i,a}(w^*)}{\partial w_1} = KL(p_1,\tilde{p}_1^{i})$ and $\frac{\partial \mathcal{P}_{i,a}(w^*)}{\partial w_i} = KL(p_i,\tilde{p}_i)$. We shall use this form of equation \ref{KL_ratio_sum_eq} to derive expressions for $w_i, \ i \in [K] \backslash \{1\}$ under the following cases:

\textbf{Case 1: $\alpha_1 \neq \alpha_{max},$} \\
\textbf{Case 2: $\alpha_1 = \alpha_{max} > \alpha_i, \ \forall i \neq 1,$} \\
\textbf{Case 3: $\alpha_1 = \alpha_2 = \alpha_{max} > \alpha_i, \ \forall i \neq 1,2,$}\\
\textbf{Case 4: $\alpha_1 = \alpha_k = \alpha_{max} \geq \alpha_i, \ i \notin \{1,2,k\}$, $\alpha_{max}>\alpha_2$ and $\zeta > 1$}\\
\textbf{Case 5: $\alpha_1 = \alpha_k = \alpha_{max} \geq \alpha_i, \ i \notin \{1,2,k\}$, $\alpha_{max}>\alpha_2$ and $\zeta \leq 1$}
\\
\\
where $\alpha_{max} := \max_i \alpha_i$. We shall first show that \textbf{\textit{Case 1}} is equivalent to $C_{1i}^{a} \to 0, C_i^{a} \to A_i(\neq 0) \forall i \neq 1$
\\
\\
For the ``if" direction, let us assume that $\alpha_1 \geq \alpha_i$ for all $i \in [K] \backslash \{1\}$. In the limit as $\gamma \to 0$, we then get that
\begin{equation*}
     \sum_{i=2}^{K}\frac{KL(p_1,\tilde{p}_1^{(i)})}{KL(p_i,\tilde{p}_i)} = \sum_{i=2}^{K}\frac{\gamma^{\alpha_1}(\sum_jp_{1j}\log(1 + C_{1i}^{a}a_{1j}) - C_{1i}^{a}\sum_ja_{1j}\tilde{p}_{1j}^{(i)})}{\gamma^{\alpha_i}(\sum_jp_{ij}\log(1 - C_{i}^{a}a_{ij}) + C_i^{a}\sum_j a_{ij}\tilde{p}_{ij})} = 1 \Rightarrow 0 = 1
\end{equation*}
which is an absurdity.\\
\\
For the ``only if" direction, let us suppose that for some $k \in [K] \backslash \{1\}$, $\alpha_1<\alpha_k$.
If $C_k^{a} \to 0$, from our analysis in Step 1, we can conclude that $C_{1k}^{a} \to A_{1k}(\neq 0)$. Therefore,
\[\gamma^{\alpha_1-\alpha_k}\frac{(\sum_j p_{1j}\log(1 + C_{1k}^{a}a_{1j}) - C_{1k}^{a}\sum_j a_{1j}\tilde{p}_{1j}^{(k)})}{(\sum_j p_{kj}\log(1 - C_{k}^{a}a_{kj}) + C_k^{a}\sum_j a_{kj}\tilde{p}_{kj})} \to \infty \  \textrm{as} \ \gamma \to 0\]
contradicting $\sum_{i=2}^{K}\frac{\gamma^{\alpha_1}(\sum_j p_{1j}\log(1 + C_{1i}^{a}a_{1j}) - C_{1i}^{a}\sum_j a_{1j}\tilde{p}_{1j}^{(i)})}{\gamma^{\alpha_i}(\sum_j p_{ij}\log(1 + C_{i}^{a}a_{ij}) + C_i^{a}\sum_j a_{ij}\tilde{p}_{ij})} =1$.
\\
\\
From our analysis in Step 1, we can conclude that $C_k^{a} \to A_k(\neq 0)$ implies that $C_{1k}^{a} \to 0$ and consequently, $C_{1i}^{a} \to 0, C_i^{a} \to A_i(\neq 0) \ \forall i \neq 1$.
\\
\\
Let $\alpha_{max} = \alpha_k$. Since $C_{1i}^{a} \to 0, C_i^{a} \to A_i(\neq 0) \ \forall i \neq 1$, we can use Taylor series expansions to write 
\begin{equation*}
\begin{aligned}
&\lim_{\gamma \to 0}\sum_{i=2}^{K}\frac{\gamma^{\alpha_1}(\sum_j p_{1j}\log(1 + C_{1i}^{a}a_{1j}) - C_{1i}^{a}\sum_j a_{1j}\tilde{p}_{1j}^{(i)})}{\gamma^{\alpha_i}(\sum_j p_{ij}\log(1 + C_{i}^{a}a_{ij}) + C_i^{a}\sum_j a_{ij}\tilde{p}_{ij})} =1\\
\Rightarrow &\lim_{\gamma \to 0}\sum_{i=2}^{K} \frac{\frac{(C^a_{1i})^2 \sum_j a_{1j}^2p_{1j}}{2}\gamma^{\alpha_1-\alpha_i}}{(\sum_j p_{ij}\log(1 + C_{i}^{a}a_{ij}) + C_i^{a}\sum_j a_{ij}\tilde{p}_{ij})} = 1\\
\end{aligned}
\end{equation*}

We know that $C_{1i}^{a} = C_i^{a} \frac{w_i\gamma^{\alpha_i}}{w_1\gamma^{\alpha_1}}$. This substitution will give us


\begin{equation*}
\begin{aligned}
    &\lim_{\gamma \to 0} \sum_{i=2}^{K} \frac{\frac{(C^a_{i})^2 \sum_j a_{1j}^2p_{1j}}{2}}{(\sum_j p_{ij}\log(1 + C_{i}^{a}a_{ij}) + C_i^{a}\sum_j a_{ij}\tilde{p}_{ij})}\bigg(\frac{w_i}{w_1}\bigg)^2\gamma^{\alpha_i - \alpha_1} = 1 \\
    \Rightarrow &\sum_{i=2}^{K}\lim_{\gamma \to 0} M_i \bigg(\frac{w_i}{w_1}\bigg)^2\gamma^{\alpha_i - \alpha_1} = 1; \textrm{where } M_i := \frac{\frac{(C^a_{i})^2 \sum_j a_{1j}^2p_{1j}}{2}}{(\sum_j p_{ij}\log(1 + C_{i}^{a}a_{ij}) + C_i^{a}\sum_j a_{ij}\tilde{p}_{ij})}
\end{aligned}
\end{equation*}

If $\alpha_i < \alpha_1$, then $\gamma^{\alpha_i - \alpha_1}$ must go to $\infty$ as $\gamma \to 0$. But $M_i$ being bounded and $M_i \big(\frac{w_i}{w_1}\big)^2\gamma^{\alpha_i - \alpha_1} \leq 1$ implies that $\frac{w_i}{w_1} \leq \frac{1}{M_i}\gamma^{\frac{\alpha_1-\alpha_i}{2}}$. Therefore, $M_i \big(\frac{w_i}{w_1}\big)^2\gamma^{\alpha_i - \alpha_1} = M_i(\frac{C_{1i}^{a}}{C_i^{a}})(\frac{w_i}{w_1}) \to 0$ as $\gamma \to 0$.
\\
\\
If $\alpha_1<\alpha_i<\alpha_{max}$, let us suppose $M_i \big(\frac{w_i}{w_1}\big)^2\gamma^{\alpha_i - \alpha_1} = M_i. \frac{C_k^{a}}{C_i^{a}}.\frac{w_k\gamma^{\alpha_k}}{w_i\gamma^{\alpha_i}}.\frac{w_i}{w_1} \to L_i \neq 0$ as $\gamma \to 0$. Let us choose an $\epsilon>0$ such that $L_i-\epsilon>0$. Then for sufficiently small $\gamma$, we get $w_k\gamma^{\alpha_k}>(L_i - \epsilon)\frac{w_1\gamma^{\alpha_1}}{w_i}$. But due to $M_k \big(\frac{w_i}{w_1}\big)^2\gamma^{\alpha_k - \alpha_1} \leq 1$, we must have $(L_i - \epsilon)^2\frac{M_k}{w_i^2}\gamma^{\alpha_1-\alpha_k}<M_k \big(\frac{w_i}{w_1}\big)^2\gamma^{\alpha_k - \alpha_1} \leq 1$. This implies that $w_i > (L_i - \epsilon) \sqrt{M_k} \gamma^{\frac{\alpha_1-\alpha_k}{2}}$. But we cannot have $w_i \to \infty$ as $\gamma \to 0$.
\\
\\
We are thus forced to conclude that only those values of $i$ for which $\alpha_i = \alpha_{max}$ will contribute positively to the sum $\sum_{i=2}^{K}\lim_{\gamma \to 0} M_i \big(\frac{w_i}{w_1}\big)^2\gamma^{\alpha_i - \alpha_1}$.
\\
\\
For $i$ such that $\alpha_i = \alpha_{max}$, as $\gamma \to 0$, let $M_i \big(\frac{w_i}{w_1}\big)^2\gamma^{\alpha_i - \alpha_1} \to L_i \neq 0$. Therefore, in the limit, $w_1 = \sqrt{\frac{M_i}{L_i}}\gamma^{\frac{\alpha_{max}-\alpha_1}{2}}w_i$. This also gives us that as $\gamma \to 0$, for all $s,t$ such that $\alpha_s = \alpha_t = \alpha_{max}$, $\frac{w_s}{w_t} = \sqrt{\frac{M_tL_s}{M_sL_t}}=\sqrt{\frac{L_s}{L_t}}\sqrt{\frac{\sum_jp_{sj}\log(1 + A_{s}a_{sj}) + A_s\sum_j a_{sj}\tilde{p}_{sj}}{\sum_j p_{tj}\log(1 + A_{t}a_{tj}) + A_t\sum_j a_{tj}\tilde{p}_{tj}}}$.
\\
\\
To approximately solve our maxmin problem, we do the following:
\\
\\
Let us fix a $k$ with $\alpha_k = \alpha_{max}$ and set $w_k = 1$. Then, $w_1 = \sqrt{\frac{M_k}{L_k}}\gamma^{\frac{\alpha_{max}-\alpha_1}{2}}$. For the other $i$ such that $\alpha_i<\alpha_{max}$, using $C_i^{a} w_i \gamma^{\alpha_i}=\frac{\sum_j a_{1j}p_{1j} + \sum_j \frac{p_kj}{A_k}\log(1 - A_ka_{kj})}{\sum_j a_{1j}p_{1j} + \sum_j \frac{p_ij}{A_i}\log(1 - A_ia_{ij})}C_k^{a} w_k \gamma^{\alpha_k}$, we get that $w_i=\frac{A_k \sum_j a_{1j}p_{1j} + \sum_j {p_{kj}}\log(1 - A_ka_{kj})}{A_i\sum_j a_{1j}p_{1j} + \sum_j {p_{ij}}\log(1 - A_ia_{ij})}\gamma^{\alpha_k - \alpha_i}$. Note that $A_i$ may be obtained by solving $\mu_1 = \sum_j \frac{a_{ij}p_{ij}}{1-A_ia_{ij}}$. For any other $s$ with $\alpha_s = \alpha_{max}$, we have $w_s = \sqrt{\frac{L_s}{L_k}}\sqrt{\frac{\sum_j p_{sj}\log(1 + A_{s}a_{sj}) + A_s\sum_ja_{sj}\tilde{p}_{sj}}{\sum_j p_{kj}\log(1 + A_{k}a_{kj}) + A_k\sum_ja_{kj}\tilde{p}_{kj}}}$. We use this to evaluate $L_k$ for each ``rarest arm" and finally normalize the weights obtained to lie within [0,1].
\\
\\
\textit{\textbf{Special case:} If there is a unique $k$ with $\alpha_k = \alpha_{max}$, then our analysis tells us that $L_k = 1$. Our approximate solution then becomes the normalized form of $w_1 = \sqrt{M_k}\gamma^{\frac{\alpha_{max}-\alpha_1}{2}}$, $w_i=\frac{A_k \sum_j a_{1j}p_{1j} + \sum_j {p_{kj}}\log(1 - A_ka_{kj})}{A_i\sum_j a_{1j}p_{1j} + \sum_j {p_{ij}}\log(1 - A_ia_{ij})}\gamma^{\alpha_k - \alpha_i}$ for $i \neq k,1$, and $w_k = 1$.}
\\
\\
Before starting on rest of the cases, we'll introduce some additional notation that will be of importance. Let us revisit the following function introduced in section \ref{approx_lb_section}.
\[g_i(x) = \bigg\{y : \sum_j \frac{a_{1j}p_{1j}}{1+ya_{1j}} = \sum_j\frac{a_{ij}p_{ij}}{1-xa_{ij}}\bigg\}\]
Clearly, $g_i$ is decreasing in $x$, and $g_k(A_k) = A_{1k}$. We now define $f_i(x)$ as 
\[f_i(x):= \sum_j p_{1j}\log(1+g_i(x)a_{1j}) + \frac{g_i(x)}{x}\sum_j p_{ij}\log(1-xa_{ij})\]
\[f_i(0) := \lim_{x \to 0^{+}} f_i(x)\]
$f_i$ can also be shown to be decreasing in $x$ and increasing in $g_i(x)$. Further, we define $h_i$ as follows.
\[h_i(x) := \frac{\sum_j p_{1j}\log(1 + g_i(x)a_{1j}) - g_i(x)\sum_j a_{1j}\tilde{p}_{1j}^{(i)}}{\sum_j p_{ij}\log(1 - xa_{ij}) + xa_{ij}\tilde{p}_{ij}}\]
It can be showed that $h_i$ is a decreasing function of $x$.
\\
\\
We can now turn our attention to \textbf{\textit{Case 2}}.
\\
\\
Since $\alpha_1=\alpha_{max}$ uniquely, in the sum
\begin{equation*}
    \sum_{i=2}^{K}\lim_{\gamma \to 0}\frac{\gamma^{\alpha_1}(\sum_j p_{1j}\log(1 + C_{1i}^{a}a_{1j}) - C_{1i}^{a}\sum_j a_{1j}\tilde{p}_{1j}^{(i)})}{\gamma^{\alpha_i}(\sum_j p_{ij}\log(1 - C_{i}^{a}a_{ij}) + C_i^{a}\sum_j a_{ij}\tilde{p}_{ij})} = 1,
\end{equation*}

if we do not have $C_k^{a} \to 0$ as $\gamma \to 0$ for some $k$, then the sum on the left becomes equal to 0, which would be a contradiction. We also note that there will be exactly one arm $k$ where $C_k^{a} \to 0$ as $\gamma \to 0$. Let us separately examine this $k^{\textrm{th}}$ summand.
\begin{equation*}
    \lim_{\gamma \to 0}\frac{(\sum_j p_{1j}\log(1 + C_{1k}^{a}a_{1j}) - C_{1k}^{a}\sum_j a_{1j}\tilde{p}_{1j}^{(i)})}{(\sum_j p_{kj}\log(1 - C_{k}^{a}a_{kj}) + C_k^{a}\sum_j a_{kj}\tilde{p}_{kj})}\gamma^{\alpha_1 - \alpha_k} = \lim_{\gamma \to 0}\frac{2(\sum_j p_{1j}\log(1 + C_{1k}^{a}a_{1j}) - C_{1k}^{a}\sum_j a_1\tilde{p}_{1j}^{(k)})}{(C^{a}_k)^2\sum_j a_{kj}^2p_{kj}}\gamma^{\alpha_1 - \alpha_k}
\end{equation*}
Since this term needs to be equal to 1, we must have
\begin{equation*}
    \lim_{\gamma \to 0} \frac{(C^{a}_k)^2}{\gamma^{\alpha_k - \alpha_1}} = \lim_{\gamma \to 0} \frac{(C^{a}_{1k})^2w_k^2\gamma^{\alpha_k - \alpha_1}}{w_1^2} = \frac{\sum_j a_{kj}^2p_{kj}}{2(\sum_j p_{1j}\log(1 + A_{1k}a_{1j}) - A_{1k}\sum_ja_{1j}\tilde{p}_{1j}^{(k)})}
\end{equation*}
This suggests the following form for $w_k$.
\begin{equation*}
    w_k = \frac{1}{A_{1k}}\sqrt{\frac{\sum_j a_{kj}^2p_{kj}}{2(\sum_j p_{1j}\log(1 + A_{1k}a_{1j}) - A_{1k}\sum_ja_{1j}\tilde{p}_{1j}^{(k)})}}w_1\gamma^{\frac{\alpha_1-\alpha_k}{2}} (=: M_k w_1\gamma^{\frac{\alpha_1-\alpha_k}{2}})
\end{equation*}
We shall now establish that $k=2$.
\\
\\
It can be understood that $g_i(x)$ is the factor by which the mean of arm 1 is reduced to $\frac{a_{ij}p_{i}}{1 - xa_i}$. Hence, we conclude that $g_2(0)<...<g_K(0)$, implying that $f_2(0)<...<f_K(0)$.
\\
\\
Observe that (8) can be expressed as (as $A_k=0$)
\begin{equation*}
    f_i(A_i) = f_k(A_k)=f_k(0)
\end{equation*}
If $k>2$, we have $f_2(A_2)<f_2(0)<f_k(0)$, giving us a contradiction. Hence, $k=2$.
\\
\\
Since for every other arm $i$, $C_{1i}^{a} \to A_{1i}(\neq 0)$ and $C_i^{a} \to A_i (\neq 0)$ as $\gamma \to 0$,
\[w_i = \frac{A_{1i}}{A_i}w_1\gamma^{\alpha_1-\alpha_i}\]
where $A_{1i}$ and $A_i$ can be obtained by finding the unique solution to
\begin{equation*}
    \frac{\sum_j p_{1j}\log(1+A_{12}a_{1j}) - A_{12}\sum_ja_{2j}p_{2j}}{\sum_j p_{1j}\log(1 + A_{1i}a_{1j}) + \frac{A_{1i}}{A_i}\sum_j p_{ij}\log(1 - A_ia_{ij})} = 1
\end{equation*}
and
\begin{equation*}
    \sum_j \frac{a_{1j}p_{1j}}{1+A_{1i}a_{1j}} = \sum_j\frac{a_{ij}p_{ij}}{1-A_{i}a_{ij}}
\end{equation*}
the latter equality following from the limit form of the mean equation. We can then use the same normalization technique as in case 1 to find the optimal weights.
\\
\\
For \textbf{\textit{Case 3}}, if $C^{a}_{12} \to A_{12} (\neq 0), C_2^{a} \to 0$ as $\gamma \to 0$, we have 
\begin{equation*}
    \lim_{\gamma \to 0}\frac{(\sum_j p_{1j}\log(1 + C^{a}_{12}a_{1j}) - C^{a}_{12}\sum_j a_{1j}\tilde{p}_{1j}^{(i)})}{(\sum_j p_{2j}\log(1 - C_{2}^{a}a_{2j}) + C_2^{a}\sum_j a_{2j}\tilde{p}_{2j})}\gamma^{\alpha_1 - \alpha_2} = \lim_{\gamma \to 0}\frac{2(\sum_j p_{1j}\log(1 + C^{a}_{12}a_{1j}) - C^{a}_{12}\sum_ja_{1j}\tilde{p}_{1j}^{(2)})}{(C^{a}_2)^2\sum_ja_{2j}^2p_2} = \infty
\end{equation*}
which is impossible, thereby guaranteeing $C^{a}_{12} \to A_{12} (\neq 0), C_2^{a} \to A_2 (\neq 0)$ as $\gamma \to 0$, and $w_2 = \frac{A_{12}}{A_2}w_1$. This will enable us to find $w_2$ as described under case 2.
\\
\\
As already argued in case 2, $C_2^{a} \to A_2 (\neq 0)$ as $\gamma \to 0$ means that $C_i^{a} \to A_i (\neq 0)$ as $\gamma \to 0$ for all $i \neq 2$. Therefore, we must have 
\begin{equation*}
    \lim_{\gamma \to 0}\frac{\sum_j p_{1j}\log(1 + C^{a}_{12}a_{1j}) - C^{a}_{12}\sum_j a_{1j}\tilde{p}_{1j}^{(i)}}{\sum_j p_{2j}\log(1 - C_{2}^{a}a_{2j}) + C_2^{a}\sum_j a_{2j}\tilde{p}_{2j}} = 1
\end{equation*}

where $A_{1i}$ and $A_i$ can be related by 
\begin{equation}\label{A_1i_A_i_relation}
    \frac{\sum_jp_{1j}\log(1 + A_{12}a_{1j}) + \frac{A_{12}}{A_2}\sum_j p_{2j}\log(1 - A_2a_{2j})}{\sum_jp_{1j}\log(1 + A_{1i}a_{1j}) + \frac{A_{1i}}{A_i}\sum_j p_{ij}\log(1 - A_ia_{ij})} = 1
\end{equation}
and using the mean equation,
\begin{equation*}
 \begin{aligned}
      \sum_j \frac{a_{1j}p_{1j}}{1+A_{1i}a_{1j}} = \sum_j\frac{a_{ij}p_{ij}}{1-A_{i}a_{ij}} \ \forall i
 \end{aligned}   
\end{equation*}

Let us denote these by $A_2(A_{12})$ and $A_i(A_{1i})$. Substituting them in \ref{A_1i_A_i_relation} and using the defintions of $f_i$, we have $f_2(A_{12})=f_i(A_{1i})$.
\\
\\
Each of these $f_i$'s is increasing in $A_{1i}$. Thus we have $A_{1i}=f_i^{-1}\circ (f_2(A_{12}))$.
\\
\\
Using this, we can solve for $A_{12}$ from equation \ref{KL_ratio_sum_eq}. We observe that each summand in \ref{KL_ratio_sum_eq} is an increasing function of $A_{1i}$ and hence $A_{12}$. So a simple efficient scheme to find the solution is to first guess an $A_{12}$ and then use a simple bisection method to numerically get $A_{1i}$'s for this guess. The mean equations can be used to get the $A_{i}$'s. Finally, we check if \ref{KL_ratio_sum_eq} is satisfied (upto tolerance). If LHS of \ref{KL_ratio_sum_eq} is greater than 1, then we halve our initial guess, and double the guess if lesser than 1. And repeat the earlier procedure till error tolerance is breached.
\\
\\
It only remains to consider \textbf{\textit{Cases 4 and 5}}. We have already argued under case 3 that $C_j^{a} \to A_j (\neq 0) \ \textrm{as} \ \gamma \to 0$ whenever $\alpha_j = \alpha_{max}$. Corresponding to any such $A_j$, we can write all other $A_i$'s in terms of $A_j$. Let us define $\xi_{ij}(x)$ as follows.
\begin{equation*}
    \xi_{ij}(x):= \bigg\{y : \frac{p_{1j}\log(1 + g_i(y)a_1) + p_{i}\frac{g_i(y)}{y}\log(1 - ya_i)}{p_{1j}\log(1 + g_j(x)a_1) + p_{j}\frac{g_j(x)}{x}\log(1 - ya_i)} = 1 \bigg\}
\end{equation*}
Let us now define $\zeta$ as
\[\zeta := \sum_{\substack{\{k: k \neq 1,\\ \alpha_k = \alpha_{max}\}}} h_k(\xi_{k2}(0)).\]
Equation \ref{KL_ratio_sum_eq} can now be re-written after taking the limit $\gamma \to 0$ as
\begin{equation*}
    \sum_{\substack{\{k: k \neq 1,\\ \alpha_k = \alpha_{max}\}}}h_k(A_k)+\lim_{\gamma \to 0}(\gamma^{\alpha_1-\alpha_2}h_2(C_2^{a}))=1
\end{equation*}

The issue now is to determine if $C_2^{a} \to 0$ as $\gamma \to 0$. We have observed earlier that $h_i(A_i)$ is a decreasing function of $A_i$ and the bijective map $\xi_{i2}$  implies $h_i(A_i)$ is also a decreasing function of $A_2$. Thus, we have 
\[\zeta \geq \sum_{\substack{\{k: k \neq 1,\\ \alpha_k = \alpha_{max}\}}}h_k(A_k).\]

If $\zeta > 1$, then equation \ref{KL_ratio_sum_eq} can be satisfied only when $C_2^{a} \to A_2 \ (>0).$ Because otherwise, the first term itself would contribute more than 1 and we'd have a contradiction. Similarly, when  $\zeta \leq 1$, we must necessarily have $C_2^{a} \to 0.$
\\
\\
In the case when $\zeta >1$, the $A_i,A_{1i}$'s are determined exactly as in 3. If $\zeta \leq 1 $ then $A_i,A_{1i}$'s are determined exactly as in Case 2. This completes our proof.

\section{The meeting point of the means in the approximate problem}\label{approx_mean_meet}

Equation (\ref{wt_kl_eq}) in the main body and the Mean Value Theorem together give us the following chain of equalities/inequalities.

\begin{equation*}
\begin{aligned}
    &\sum_{j=1}^{n}p_{1j}\log(1+C_{1s}a_{1j}) - C_{1s}\tilde{\mu}_s \\
    \leq &\sum_{j=1}^{n}p_{1j}\log(1+C_{1s}a_{1j}) - C_{1s}\sum_{j=1}^{n}\frac{a_{sj}p_{sj}}{1-C_{s}a_{sj}}\\
    \leq &\sum_{j=1}^{n}p_{1j}\log(1+C_{1s}a_{1j})+\frac{C_{1s}}{C_s}\sum_{j=1}^{n}p_{sj}\log(1-C_{s}a_{sj})\\
    = &\sum_{j=1}^{n}p_{1j}\log(1+C_{1t}a_{1j})+\frac{C_{1t}}{C_t}\sum_{j=1}^{n}p_{tj}\log(1-C_{t}a_{tj})\\
    \leq &\sum_{j=1}^{n}p_{1j}\log(1+C_{1t}a_{1j}) - C_{1t}\mu_t \\
\end{aligned}
\end{equation*}
Regrouping terms among the first and last quantities of the above chain gives us that
\[\frac{C_{1t}}{C_{1s}}\mu_t \leq \frac{1}{C_{1s}}\sum_{j=1}^{n}p_{1j}\log\bigg(\frac{1+C_{1t}a_{1j}}{1+C_{1s}a_{1j}}\bigg) + \tilde{\mu}_s\]
Note that $\log\big(\frac{1+C_{1t}a_{1j}}{1+C_{1s}a_{1j}}\big) = \log\big(1 + \frac{(C_{1t}-C_{1s})a_{1j}}{1+C_{1s}a_{1j}}\big) \leq (C_{1t}-C_{1s})\tilde{\mu}_s$, and hence, $\frac{C_{1t}}{C_{1s}}\mu_t \leq \frac{C_{1t}}{C_{1s}}\tilde{\mu}_s$, i.e., $\mu_t \leq \tilde{\mu}_s$.
\\
\\
We conclude from the above analysis that $\forall s,t \neq 2$, $\tilde{\mu}_s \geq \mu_t \Rightarrow \forall s \neq 2$, $\tilde{\mu}_s \geq \mu_2$.

\section{Proof of $\delta$-Correctness of TS(A).}
Let the set of all possible bandit hypotheses be $\mathcal{H}$. We have $\mathcal{H}=\cup_i \mathcal{H}_i$, where $\mathcal{H}_i$ denotes all bandit instances with arm $i$ having the highest mean. Let $\hat{i}(\tau_{\delta})$ denote the recommendation of TS(A) at the stopping time. The error probability for a bandit instance $p$ with arm 1 having the highest mean is given by:
\begin{equation*}
\begin{aligned}
     \mathbb{P}_{p}(\tau_{\delta}<\infty,\hat{i}(\tau_{\delta}) \neq 1) &\leq \mathbb{P}_{p}(\exists t \in \mathbb{N}: \hat{i}(t) \neq 1, Z_{\hat{i}(t)}(t)>\beta(t,\delta))\\
     &=\mathbb{P}_{p}(\exists t \in \mathbb{N}: \exists i \neq 1 A(\hat{p}) \subseteq \mathcal{H}_i)
\end{aligned}
\end{equation*}
where $A(\hat{p}):= \{ p' \in \mathcal{H} | \underset{b \neq \hat{i}(t)}{\min} N_{\hat{i}(t)}(t)\mathcal{K}_{inf}^{L}(\hat{p}_{\hat{i}(t)}(t),\mu_{\hat{i}(t)}')+N_{b}(t)\mathcal{K}_{inf}^{U}(\hat{p}_{b}(t),\mu_{b}') \leq \beta(t,\delta)\}$. This implies:
\begin{equation}\label{delta_correct_proof}
\begin{aligned}
  \mathbb{P}_{p}(\tau_{\delta}<\infty,\hat{i}(\tau_{\delta}) \neq 1) &\leq \mathbb{P}_{p}(\exists t \in \mathbb{N} : p \notin A(\hat{p}))\\ 
  &=\mathbb{P}_{p}(\exists t \in \mathbb{N}: \underset{b \neq \hat{i}(t)}{min} N_{\hat{i}(t)}(t)\mathcal{K}_{inf}^{L}(\hat{p}_{\hat{i}(t)}(t),\mu_{\hat{i}(t)})+N_{b}(t)\mathcal{K}_{inf}^{U}(\hat{p}_{b}(t),\mu_{b}) \geq \beta(t,\delta))\\
  & \leq \sum_{b \neq 1}\mathbb{P}_{p}(\exists t \in \mathbb{N}: N_{\hat{i}(t)}(t)\mathcal{K}_{inf}^{L}(\hat{p}_{\hat{i}(t)}(t),\mu_{\hat{i}(t)})+N_{b}(t)\mathcal{K}_{inf}^{U}(\hat{p}_{b}(t),\mu_{b}) \geq \beta(t,\delta))
\end{aligned}
\end{equation}
Now a concentration inequality for the above quantity was shown in \cite{agrawal2021optimal}.
\begin{proposition}{4.2 in \cite{agrawal2021optimal}} 
\begin{equation*}
    \mathbb{P}\bigg(\exists n \in \mathbb{N}: N_{i}(n)\mathcal{K}_{inf}^{U}(\hat{p}_{i}(t),\mu_{i})+\mathcal{K}_{inf}^{L}(\hat{p}_{j}(t),\mu_{j})\geq x+5 \log(n+1)+2\bigg) \leq e^{-x}.
\end{equation*}
\end{proposition}
Substituting this in (\ref{delta_correct_proof}) finishes the proof.
\section{Sample complexity guarantee for TS(A).}
We follow closely the section C.6.2 in \cite{agrawal2020optimal}.
Let $\hat{w}^*(p)$ denote the optimal weights obtained as solutions to the approximate problem described at the beginning of section \ref{approx_lb_section} in the main paper. Lemma 14 in \cite{agrawal2020optimal} then tells us that TS(A) ensures that for all arms $i \in [K]$, $\frac{N_i(lm)}{lm} \overset{a.s.}{\to}\hat{w}^*(p)$ as $l \to \infty$. Recall from section \ref{approx_algo_section} of the main paper that $l$ is the batch index and $m$ is the batch size.

Define the following set 
\begin{equation*}
\mathcal{I}_{\epsilon}(p) := B_{\zeta}(p_1) \times ... \times B_{\zeta}(p_K)
\end{equation*}
where
\begin{equation*}
B_{\zeta}(p_i) := \{\tilde{p}_i : d_W(p_i,\tilde{p}_i) \leq \zeta, |\tilde{\mu}_i - \mu_i| \leq \zeta\}.
\end{equation*}
Here, $d_W$ is the Wasserstein-1 metric on probability measures and $\tilde{\mu}_i$ is the mean of $\tilde{p}_i$.\\
Whenever the empirical bandit $\hat{p}(lm) \in \mathcal{I}_{\epsilon}(p)$, arm1 becomes empirically best. For $\epsilon >0$, choose $\zeta : = \zeta(\epsilon)(<\frac{\mu_1-\mu_2}{4})$ such that 
\begin{equation*}
    \underset{i \in [K]}{\max}|\hat{w}^*_i(p')-\hat{w}_i^*(p)| \leq \epsilon
\end{equation*}
for all $p' \in \mathcal{I}_{\epsilon}(p)$. For $ T \in \mathbb{N}$, $T \geq m$, define $\ell_0(T):= \max \{1,\frac{T^{1/4}}{m} \}$, $\ell_1(T):= \max \{1,\frac{T^{3/4}}{m} \}$ and $\ell_2(T):= \lfloor \frac{T}{m} \rfloor$. Define the following set 
\begin{equation*}
    \mathcal{G}_T(\epsilon):= \bigcap^{\ell_2(T)}_{l=\ell_0(T)}\{\hat{p}(lm) \in \mathcal{I}_{\epsilon}(p)\}\bigcap^{\ell_2(T)}_{l=\ell_1(T)}\bigg\{\underset{i \in [K]}{\max}\bigg|\frac{N_i(lm)}{lm}-\hat{w}_i^*(p)\bigg| \leq \epsilon \bigg\}
\end{equation*}
Define the quantities:
\begin{equation*}
\begin{aligned}
    \tilde{g}(p,w) &:= \underset{b \neq 1}{\min} \mathcal{P}_{b}(w)\\
    \tilde{C}_\epsilon (p) &:= \underset{\underset{\{w^{'} : ||w^{'}-\hat{w}^*(p)|| \leq \epsilon \}}{p^{'} \in \mathcal{I}_{\epsilon}(p)}}{inf }\tilde{g}(p^{'},w^{'}).
\end{aligned}
\end{equation*}
where $\mathcal{P}_{b}$ was defined in equation \ref{P_i_def} of the main paper. Now the stopping rule (see section \ref{approx_algo_section} in the main paper) is given by:
\begin{equation*}
    Z_{k^*}(l)>\beta(lm,\delta)
\end{equation*}
where
\begin{equation*}
\begin{aligned}
    Z_{k^*}(l) := &\underset{b \neq k^*}{\min}\underset{x \leq y}{\inf} N_{k^*}(lm)\mathcal{K}_{inf}^{L}(\hat{p}_{k^*}(lm),x) \\
    &+ N_{b}(lm)\mathcal{K}_{inf}^{U}(\hat{p}_b(lm),y).
\end{aligned}
\end{equation*}
where $k^*$ is the empirical best arm and $\beta(t,\delta)$ is the stopping threshold defined as
\begin{equation*}
    \beta(t,\delta) := \log\bigg(\frac{K-1}{\delta}\bigg)+5\log(t+1)+2.
\end{equation*}
Note that in $\mathcal{G}_T(\epsilon)$ we have $Z_{k^*}(l)> lm \times \tilde{C}_\epsilon (p)$. Hence, in $\mathcal{G}_T(\epsilon)$,
\begin{equation*}
\begin{aligned}
    \min\{\tau_\delta, T\} & \leq m.l_1(T) + m\sum_{l=l_1(T)+1}^{l_2(T)}\mathbb{I}\{lm<\tau_\delta\}\\
    & \leq m.l_1(T) + m\sum_{l=l_1(T)+1}^{l_2(T)}\mathbb{I}\{Z_{k^*}(l)<\beta(lm,\delta)\}\\
    & = m.l_1(T) + m\sum_{l=l_1(T)+1}^{l_2(T)}\mathbb{I}\bigg\{l<\frac{\beta(lm,\delta)}{m\tilde{C}_{\epsilon}(p)}\bigg\}\\
    & = m.l_1(T) + \frac{\beta(T,\delta)}{\tilde{C}_{\epsilon}(p)} 
\end{aligned}
\end{equation*}
Define $T_0(\delta,\epsilon) := \inf\bigg\{t: m.l_1(T) + \frac{\beta(t,\delta)}{\tilde{C}_{\epsilon}(p)} \leq t\bigg\}$.
\\
On $\mathcal{G}_T(\epsilon)$, for $T\geq \max\{m,T_0(\delta,\epsilon)\}$, $\min\{\tau_\delta,T\} \leq T$, meaning that for such $T$, $\tau_\delta \leq T$. Hence, choosing $T_1(\delta,\epsilon) := \max\{m,T_0(\delta,\epsilon) + 1$, we get that $\mathcal{G}_{T_1(\delta,\epsilon)}(\epsilon) \subseteq \{\tau_\delta \leq T_1(\delta,\epsilon)\}$. Then, $\min\{\tau_\delta,T_1(\delta,\epsilon)\} \leq T_1(\delta,\epsilon) \Rightarrow \tau_\delta \leq T_1(\delta,\epsilon)$. This allows us to conclude that
\begin{equation*}
\begin{aligned}
    \mathbb{E}(\tau_\delta) & = \sum_{t=1}^{\infty}\mathbb{P}(\tau_\delta\geq t)\\
    & = \sum_{t=1}^{T_1(\delta,\epsilon)}\mathbb{P}(\tau_\delta \geq t) + \sum_{t=T_1(\delta,\epsilon) + 1}^{\infty}\mathbb{P}(\tau_\delta \geq t)\\
    & \leq T_0(\delta,\epsilon) + m + \sum_{t=m+1}^{\infty}\mathbb{P}(\mathcal{G}_T^{C}(\epsilon))
\end{aligned}
\end{equation*}
Now in the same manner as in \cite{agrawal2020optimal} we can show that $\frac{T_0(\delta,\epsilon)}{\log(1/\delta)} \to \frac{1}{\tilde{C}_{\epsilon}(p)}$ as $\delta \to 0$.
We invoke Lemma 32 in \cite{agrawal2020optimal} to observe that $\frac{\sum_{t=m+1}^{\infty}\mathbb{P}(\mathcal{G}_T^{C}(\epsilon))}{\log(1/\delta)} \to 0$. Thus we have for small enough $\epsilon>0$
\begin{equation*}
  \limsup_{\delta \to 0}\frac{\mathbb{E}(\tau_\delta)}{\log(1/\delta)}\leq \frac{1}{\tilde{C}_{\epsilon}(p)}
\end{equation*}
But we observe that by continuity in $\epsilon$, when $\epsilon \to 0$
\begin{equation*}
  \tilde{C}_{\epsilon}(p) \to  \underset{b \neq 1}{\min} \mathcal{P}_{b}(\hat{w}^*). 
\end{equation*}
Note by definition $\underset{b \neq 1}{\min} \mathcal{P}_{b}(\hat{w}^*) \leq V^*(p)$. This inequality shows that TS(A) suffers an increase in sample complexity but this is expected to be small when $\gamma$ is close to zero since then $\hat{w}^*(p) \approx w^*(p).$
\section{Algorithms in Literature}\label{existing_algos}
The following algorithm as per \href{https://jmlr.csail.mit.edu/papers/volume7/evendar06a/evendar06a.pdf }{Even-Dar, Mannor \& Mansour (2006)} provides a simplistic approach towards solving our problem, despite being highly expensive in terms of sampling complexity.
\begin{algorithm}
   \caption{Succesive elimination ($\delta$)}
   \begin{algorithmic}
   \STATE Set $t=1, S=[K]$.\\
   \STATE For all $i \in [K]$, set the empirical means $\hat{\mu}_i^t = 0$.

   \WHILE{|S|>1}
        \STATE Sample every arm once, update $\hat{\mu}_i^t$.
        \STATE Define $\hat{\mu}_{max}^t := \underset{i \in S}{\max}\hat{\mu}_i^t$, $\xi_t := \sqrt{\frac{\log(4Kt^2/\delta)}{t}}$.
        \STATE For all $i \in S$ such that $\hat{\mu}_{max}^t - \hat{\mu}_i^t \geq 2\xi_t$, set $S = S \backslash {i}$.
        \STATE $t = t + 1$
   \ENDWHILE
   \STATE Declare the surviving arm as the best arm.
   \end{algorithmic}
\end{algorithm}
\FloatBarrier

The successive elimination algorithm performs poorly in the rare event setting because a less rare arm which does not have the largest mean becomes likely to survive the elimination and be declared the winner. This is because the less rare arm is likely to produce a nonzero sample, thereby raising its empirical mean, while the more rare arms are yet to turn out any non-zero samples.
\\
\\
\cite{shubhada2019} describes the following algorithm to meet the lower bound on sampling complexity.

\begin{algorithm}
   \caption{Track and Stop}
   \begin{algorithmic}
   \STATE Generate $\lfloor \frac{m}{k} \rfloor$ samples for each arm.\\
   \STATE Set $l=1$. $lm$ denotes the number of samples.
   \STATE Compute the empirical bandit $\hat{\mu}=(\hat{\mu})_{\{a \in [K]}$.
   \STATE Compute the approximate weights $\hat{w}(\hat{\mu})$.
   \STATE Let $k^{*}= \underset{a \in [K]}{\arg \max } \mathbb{E}[\hat{\mu}_a]$. 
   \STATE Compute ${Z}(k^{*},l,\hat{\mu})$, $\beta(lm,\delta)$.
   \WHILE{ $l \leq 2$ or ${Z}(k^{*},l,\hat{\mu}) \geq \beta(lm,\delta)$}
        \STATE Compute $s_a=(\sqrt{(l+1)m}-N_a(lm))^+$.
       \IF {$m \geq \sum_{a}s_a$}
           \STATE Generate $s_a$ many samples for each arm $a$.
           \STATE Generate $(m-\sum_a s_a)^+$ independent samples from $\hat{w}(\hat{\mu})$. Let $Count(a)$ be occurrence of $a$ in these samples. 
           \STATE Generate $Count(a)$ samples from each arm $a$.
           \ELSE
           \STATE Solve the load balancing problem $\text{minimize } \max_a(s_a-\hat{s}_a)$, where $s_a \geq \hat{s}_a \geq 0$. 
           \STATE Generate $\hat{s}_a$ samples from each arm $a$.
       \ENDIF
           \STATE $l = l+1$
           \STATE Update empirical bandit $\hat{\mu}$ with new samples.
           \STATE Update ${Z}(k^{*},l,\hat{\mu})$, $\beta(lm,\delta)$ and $\hat{w}(\hat{\mu})$ .
   \ENDWHILE
   \STATE Declare $k^{*}$ arm as the best arm.
   \end{algorithmic}
\end{algorithm}
\FloatBarrier
\end{document}